\documentclass{article} % For LaTeX2e
\usepackage{iclr2027_conference,times}

\usepackage{amsmath,amsfonts,bm}

\def\eqref#1{equation~\ref{#1}}
\def\1{\bm{1}}

\def\eps{{\epsilon}}

\def\rmI{{\mathbf{I}}}

\def\veps{{\bm{\eps}}}

\def\vu{{\bm{u}}}
\def\vv{{\bm{v}}}
\def\vw{{\bm{w}}}
\def\vx{{\bm{x}}}

\def\vz{{\bm{z}}}

\DeclareMathAlphabet{\mathsfit}{\encodingdefault}{\sfdefault}{m}{sl}
\SetMathAlphabet{\mathsfit}{bold}{\encodingdefault}{\sfdefault}{bx}{n}

\def\gN{{\mathcal{N}}}

\def\sR{{\mathbb{R}}}

\usepackage{hyperref}
\usepackage{url}
\usepackage[utf8]{inputenc} % allow utf-8 input
\usepackage[T1]{fontenc}    % use 8-bit T1 fonts
\usepackage{booktabs}       % professional-quality tables
\usepackage{amsfonts}       % blackboard math symbols
\usepackage{nicefrac}       % compact symbols for 1/2, etc.
\usepackage{microtype}      % microtypography

\usepackage{multirow}
\usepackage{makecell}
\usepackage[table]{xcolor}
\usepackage{graphicx}
\usepackage{caption}
\usepackage{wrapfig}
\usepackage{algorithm}
\usepackage{algpseudocode}
\usepackage{thmtools,thm-restate}
\usepackage{amsthm, amssymb}
\usepackage{colortbl}
\usepackage{enumitem}

\newtheorem{lemma}{Lemma}

\def\eqref#1{Eq.~(\ref{#1})}

\title{Gradient-Free Noise Optimization for Reward \\ Alignment in Generative Models}

\iclrfinalcopy
\author{Jeongsol Kim$^{1*}$, Hongeun Kim$^{1*}$, Jian Wang$^2$, Jong Chul Ye$^1$ \\
$^1$KAIST AI, $^2$Snap Inc. \\
\small\texttt{\{jeongsol, hongeun\}@kaist.ac.kr, jwang4@snap.com, jong.ye@kaist.ac.kr} \\
$^*$Equal contribution}

\begin{document}

\maketitle
\lhead{Preprint}

\begin{figure}[!h]
    \centering
    \vspace{-0.8cm}\includegraphics[width=1.0\linewidth,height=8.3cm]{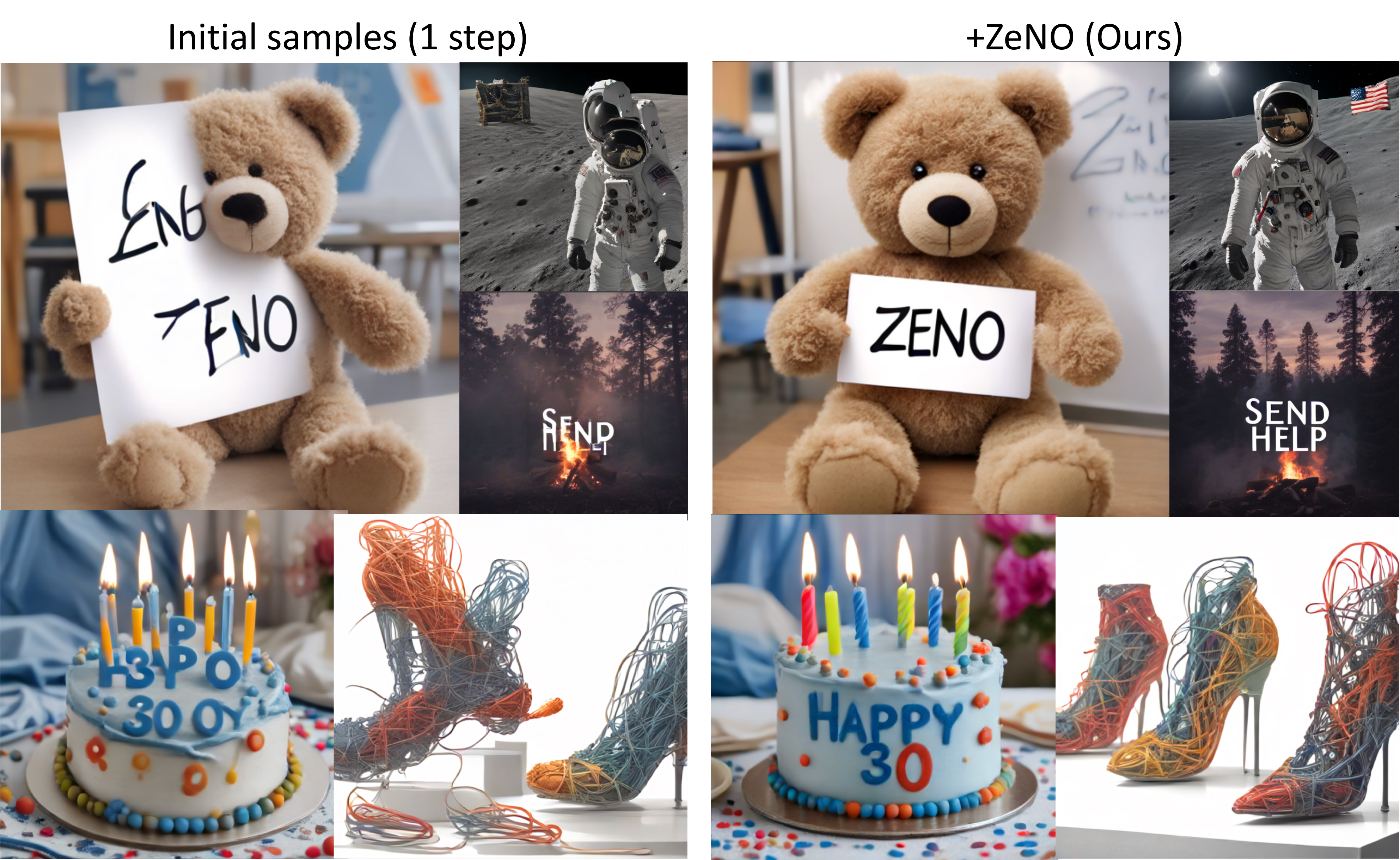}
    \caption{Representative results of ZeNO across diverse one-step generators and reward functions. Our noise optimization consistently improves image quality and prompt alignment without generator modification or backpropagation. Generator specification and the corresponding text prompts are reported in Appendix~\ref{sec:detail}.}
    \label{fig:main}
\end{figure}

\begin{abstract}
Existing reward alignment methods for diffusion and flow models rely on multi-step stochastic trajectories, making them difficult to extend to deterministic generators.
A natural alternative is noise-space optimization, but existing approaches require backpropagation through the generator and reward pipeline, limiting applicability to differentiable settings. To address this, here we present ZeNO (\textbf{Ze}roth-order \textbf{No}ise \textbf{O}ptimization), a gradient-free framework that formulates noise optimization as a path-integral control problem, estimable from zeroth-order reward evaluations. When instantiated with an Ornstein--Uhlenbeck reference process and a one-step receding horizon, the repeated update connects to Langevin dynamics targeting a reward-tilted distribution. ZeNO enables effective inference-time scaling and demonstrates strong performance across diverse generators and reward functions, including a protein structure generation task where backpropagation is infeasible.
\end{abstract}

\section{Introduction}

Post-training aims to align a pre-trained generative model with desired objectives, often specified by a reward function such as prompt fidelity, aesthetic quality,
human preference, etc.
A broad class of recent post-training methods improves diffusion or flow models by optimizing them toward such rewards while keeping them close to a reference model.
For example, DDPO~\citep{black2024training}, Flow-GRPO~\citep{liu2026flowgrpo}, PCPO~\citep{lee2026pcpo}, and SQDF~\citep{kang2026diffusion} formulate alignment from a sequential Markov decision process perspective, whereas VGG-Flow~\citep{liu2026value} and Residual $\nabla$-DB~\citep{liu2025efficient} are motivated by optimal control and GFlowNet-style principles.

Despite their different motivations, these methods are primarily developed for settings in which generation is represented as a multi-step stochastic trajectory.
This trajectory-level structure is central to their formulation, and consequently, extending these approaches to deterministic or one-step generators is not straightforward, since such generators do not naturally expose the same stochastic trajectory structure.
Recent advances in generative modeling have led to the development of few-step or one-step deterministic generators, such as consistency models~\citep{song2023consistency, luo2023latent} and distilled diffusion models~\citep{yin2024one, yin2024improved, sauer2024adversarial, geng2025mean}, which achieve high-quality generation at a fraction of the computational cost of their multi-step counterparts~\citep{yin2024one}.
Aligning these models with reward functions is therefore of growing practical importance, yet existing post-training methods are not directly applicable to this.

A natural alternative is to optimize the input noise of a deterministic generator directly~\citep{eyring2024reno,min2026origen}.
Given a generator $G_\theta$ and a reward function $r$, one could backpropagate the terminal reward through the generator and perform gradient ascent in the noise space.
However, this approach requires differentiating through the generator and reward pipeline, and can be unreliable when gradients reflect spurious local sensitivities rather than meaningful semantic improvement.
Moreover, it is inapplicable when the reward is non-differentiable or when the generator is too large to support efficient backpropagation.
For example, ReNO~\citep{eyring2024reno} and ORIGEN~\citep{min2026origen} operate in the noise space and connect noise optimization to gradient ascent or Langevin dynamics targeting a reward-tilted distribution, but
still require backpropagation through the reward-generator composition, limiting their applicability to differentiable pipelines.

% In this paper, we retain the noise-space perspective but replace direct gradient ascent with a zeroth-order stochastic control formulation, which we call \textbf{ZeNO} (\textbf{Ze}roth-order \textbf{No}ise \textbf{O}ptimization).
%
% Specifically, we treat noise evolution as a controllable process and optimize it via a path-integral control (PIC) formulation, using the final sample reward as the objective, which yields a practical zeroth-order estimator.
%
% In particular, we adopt the Ornstein-Uhlenbeck (OU) process as the reference dynamics, which provides initial noises compatible with the pre-trained generator.
%
% The PIC formulation situates this estimator within a controlled SDE framework, connecting the noise update to Langevin sampling targeting the reward-tilted distribution $p^\star(\vz) \propto p(\vz)\exp(r(G_\theta(\vz))/\lambda)$ without requiring gradients through the generator or reward model.
%
% We further introduce a variance reduction approximation that achieves significantly higher reward improvement under the same compute budget.

In this paper, we retain the noise-space perspective but recast noise optimization as a path-integral control (PIC) problem on the initial noise space of a deterministic generator, in which a control steers the noise, evolving under a reference process, toward high-reward samples.
We call the resulting framework \textbf{ZeNO} (\textbf{Ze}roth-order \textbf{No}ise \textbf{O}ptimization), which brings three properties.
First, the optimal control depends only on reward evaluations, requiring no gradient of the generator or the reward.
Second, since the generator maps noise to a clean sample in one step, the terminal cost is always evaluated on a clean sample, avoiding the Tweedie-type approximation of PIC-style guidance along diffusion trajectories.
Third, an Ornstein--Uhlenbeck (OU) reference process keeps the noise compatible with the pre-trained generator, and repeating the one-step receding-horizon control amounts to Langevin sampling of a reward-tilted distribution.
In practice, we replace the exponential path-integral estimator, which degenerates with finitely many samples, with a linearized estimator whose error is provably bounded and which yields significantly higher reward under the same compute budget.

Empirically, ZeNO achieves competitive or superior performance compared to gradient-based methods across a range of generative models and reward functions, including non-differentiable reward, and enables effective inference-time scaling as more computation is allocated.
We further validate the broad applicability of ZeNO on a protein structure generation,the designability reward is a non-differentiable pipeline of inverse folding,
refolding, and structural alignment, so that gradient-based methods are inapplicable, demonstrating strong performance even under this challenging constraint.

\section{Background}

\textbf{Stochastic Differential Equation.}
For a stochastic process $\vx_t$, a stochastic differential equation (SDE) is written as
\begin{equation}
    d\vx_t = f(\vx_t,t) dt + g(t) d\vw_t
\end{equation}
where $f(\vx_t, t)$ and $g(t)$ denote the drift and diffusion coefficients, and $\vw_t$ denotes a Wiener process.
A representative example is the Ornstein--Uhlenbeck (OU) process, obtained when $f(\vx_t, t) = -a\vx_t$ and $g(t)=b$ for $a>0$ and $b>0$.
This process has the Gaussian stationary distribution $\gN(0, b^2/2a)$~\citep{oksendal2003stochastic}.
In particular, when $2a=b^2$, the stationary distribution becomes $\gN(0, \rmI)$.

\textbf{Reward Tilted Distribution.}
In post-training generative models to maximize rewards, the alignment problem is often formulated as an entropy-regularized reinforcement learning objective over the generative trajectory.
Let $P^{\mathrm{ref}}$ denote the trajectory distribution induced by the pretrained generative model, and let $Q^\theta$ denote the trajectory distribution induced by
the finetuned model parameterized by $\theta$.
The objective can then be written as
\begin{equation}
\label{eq:softrl}
    \mathcal{L}(\theta) = \mathbb{E}_{\vx_0 \sim q_0^\theta} \left[ r(\vx_0) \right]-\lambda D_{\mathrm{KL}} \left( Q^\theta \,\|\, P^{\mathrm{ref}} \right),
\end{equation}
where the first term encourages the model to generate high-reward samples, while the KL regularization term penalizes excessive deviation from the pretrained model.
The coefficient $\lambda>0$ controls the trade-off between reward maximization and staying close to the reference generative process.
Equivalently, this KL term can often be decomposed into a sum or integral of per-step policy divergences along the sampling trajectory, depending on the specific generative process.
The marginal distribution induced by the optimal policy of \eqref{eq:softrl} is given by~\citep{uehara2024fine}
\begin{equation}
    p_t^\star(\vx_t) \propto p_t^{\mathrm{ref}}(\vx_t) e^{V(\vx_t)},\quad
    V(\vx_t):=\log  \mathbb{E}_{P^{\mathrm{ref}}} \left[ \exp\left(\frac{r(\vx_0)}{\lambda}\right) \middle| \vx_t \right],
    \label{eq:tilted}
\end{equation}
for $t\in[0,1]$, where $\vx_0$ refers to the clean sample.
This shows that the optimal marginal is obtained by tilting the reference marginal according to the conditional desirability of future high-reward samples.
Notably, this tilted form also applies at $t=1$, corresponding to the initial noise distribution, which is particularly relevant to settings where optimization is performed directly in the initial noise space.

\textbf{Path Integral Control.}
Let $\vu(\vx_t, t)$ denote a control input that guides the reference SDE
\begin{equation}
\label{eq:control_noise_sde}
d\vx_t = f(\vx_t, t)dt + g(t)\left( \vu(\vx_t,t)dt + d\vw_t \right).
\end{equation}
We seek a control that steers the trajectory toward high-reward terminal states while regularizing control effort.
To this end, we consider the terminal-cost optimal control objective
\begin{equation}
\label{eq:optimal_control_cost}
J_\vu (\vx_t, t) = \mathbb{E} \left[ m(\vx_T) + \lambda \int_t^T \frac{1}{2} \| \vu(\vx_\tau, \tau) \|^2 d\tau \right],
\end{equation}
where $m(\vx_T)$ denotes the terminal cost and the running cost penalizes the control energy. Weighted by $\lambda$, this running cost equals the KL divergence between the controlled and reference path measures, so \eqref{eq:optimal_control_cost} is the control form of the KL-regularized objective in \eqref{eq:softrl}.
Since the objective contains no intermediate-state cost, the optimal policy is driven solely by the terminal desirability of the trajectory subject to minimum
control effort.
Under this formulation, the optimal control has the path-integral form~\citep{pathintegral,pmlr-v235-huang24g, Uehara2024understanding}.
\begin{align}
\label{eq:optimal_control}
\vu^\star(\vx_t, t)= \frac{\mathbb{E}_{\vx_{t:T}} \left[ \exp\left( -m(\vx_T) /\lambda\right) d\vw_t \right]}{\mathbb{E}_{\vx_{t:T}} \left[ \exp\left(-m(\vx_T)/\lambda\right)\right]},
\end{align} %ㅎㅇ
where the expectation is taken over trajectories of the uncontrolled reference SDE in \eqref{eq:control_noise_sde}.
This shows that the control can be estimated from reward-weighted trajectory samples and depends only on terminal-cost evaluations, without requiring differentiation through the reward or the generator.

A practical difficulty is that evaluating $\vu^\star$ over the full horizon to $T$ requires repeated rollout of \eqref{eq:control_noise_sde}, which becomes costly
in long-horizon settings.
A standard approximation is to truncate the horizon to $H < T$ and replace $m(\vx_T)$ with a surrogate terminal cost evaluated after a shorter rollout.
Although this reduces computation, it is valid only when the intermediate state $X_{t+H}$ remains in the domain on which $m$ is defined.
In diffusion models, however, $X_{t+H}$ is generally a noisy latent state, whereas reward models are typically defined on clean samples.
This mismatch makes naive truncation inappropriate, and prior approaches therefore rely on Tweedie-type estimators to map noisy intermediate states to an estimate
of the underlying clean sample~\citep{chung2023diffusion,Kim_2025_ICCV, kim2025testtime}.

In contrast, ZeNO avoids this entirely by operating directly in the initial noise space, where the generator maps noise to clean samples in a single deterministic step, ensuring the terminal cost is always evaluated on a valid clean sample.

\section{Method}
We aim to align a deterministic generator with a reward function that may be non-differentiable.
Since one-step deterministic generators do not fit the multi-step policy-learning framework of diffusion post-training~\citep{black2024training,liu2025efficient,liu2026value,liu2026flowgrpo}, and their output is fully determined by the initial noise, we formulate alignment as a stochastic control problem in the input noise space whose optimal update can be estimated without gradients.

\begin{figure}[t]
    \centering
    \includegraphics[width=0.8\linewidth]{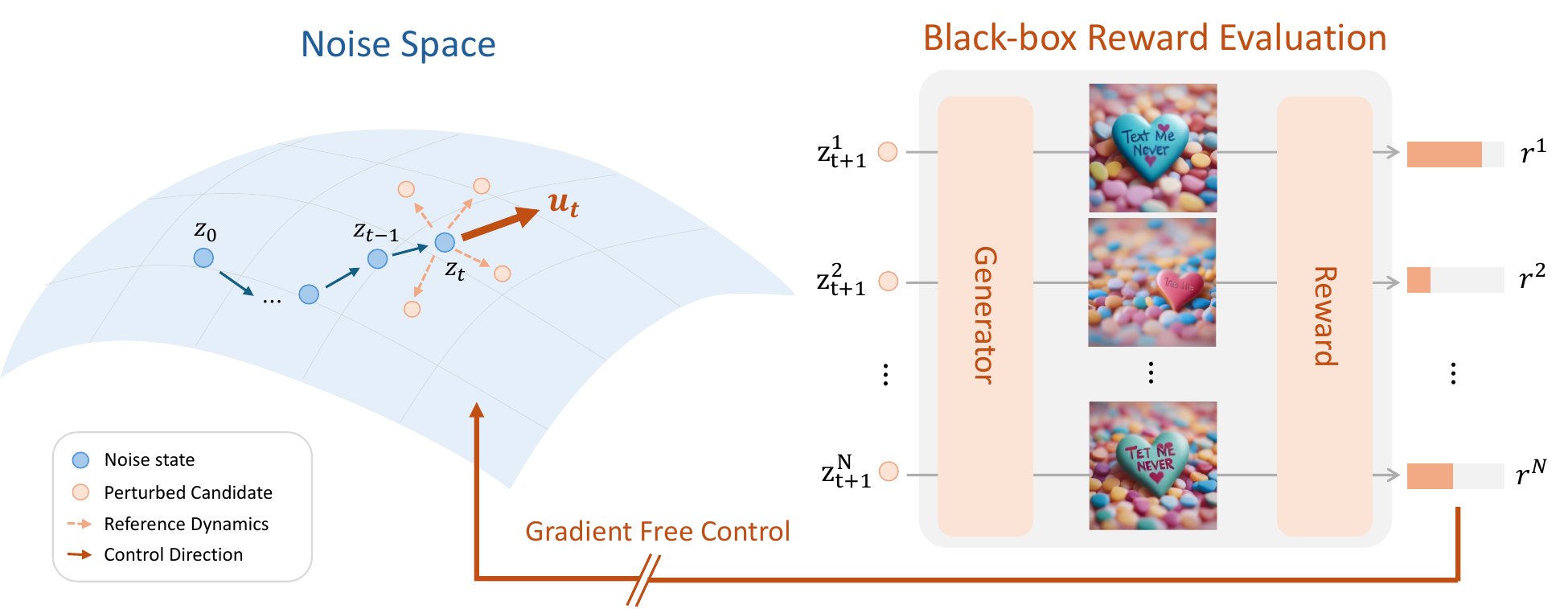}
    \caption{Overview of the ZeNO framework. At each iteration, $N$ perturbed noise candidates $\vz_{t+1}^{1:N}$ are sampled from the reference dynamics around the current noise state $\vz_t$ and evaluated by the reward model through the generator $G_\theta$. The reward-weighted perturbations are aggregated to estimate the control direction $\vu_t$, which steers the noise toward higher-reward regions without requiring gradients through the generator or the reward model.}
    \label{fig:method}
    \vspace{-0.3cm}
\end{figure}

\subsection{Problem Definition}
Let $\vz_0 \sim \mathcal{N}(0,\mathbf{I})$ denote the initial noise of a deterministic generator $G_\theta: \mathbb{R}^d \times \mathcal{C} \rightarrow \mathbb{R}^d$ with conditions $\mathcal{C}$ such as text embeddings, which we write as $G_\theta(\vz)$ for notational simplicity.
Given a reward function $r:\mathbb{R}^d \rightarrow \mathbb{R}$, our goal is to construct a noise trajectory $\{\vz_t\}_{t=0}^T$ from $\vz_0$ whose terminal state $\vz_T$ achieves a high reward.
We formalize this by introducing a controlled SDE on the noise space and defining the terminal cost as $m(\vz_T) := - r(G_\theta(\vz_T))$.

\subsection{Derivation}
\label{sec:zeno}
When we solve the path-integral control problem and obtain a control vector $\vu(\vz,t)$, the initial noise is updated by the controlled SDE:
\begin{equation}
\label{eq:noise_sde}
    d\vz = f(\vz,t)dt + g(t)\left(\vu(\vz,t)dt + d\vw_t\right),
\end{equation}
where $f(\vz,t)$ and $g(t)$ denote the drift and diffusion coefficients, respectively, $d\vw_t$ is a Wiener process, and $\vu(\vz,t)$ is the control vector.
In the absence of control, i.e., $\vu(\vz,t)=0$, this reduces to the uncontrolled reference dynamics.
Given the terminal cost $m(\vz_T)$, the optimal control becomes:
\begin{equation}
\label{eq:control_exp}
\vu^\star(\vz,t) = \frac{ \mathbb{E}_{\vz_{t:T}} \left[\exp\left(r(G_\theta(\vz_T))/\lambda\right)d\vw_t\right]}{\mathbb{E}_{\vz_{t:T}}\left[\exp\left(r(G_\theta(\vz_T))/\lambda\right)\right]},
\end{equation}
where the expectation is over trajectories induced by the uncontrolled reference dynamics, and $\vu^\star$ here denotes the control increment over $dt$, whose scale is absorbed into $\eta$.

The path-integral control suffers from high variance due to the exponential weighting, and its Monte Carlo estimate is self-normalized, biased, and loses effective samples~\citep{mcbook}.
As a remedy, we subtract the mean reward as a baseline and linearize the exponential weight around it, which removes the self-normalization and yields the following estimator and error bound (proof in Appendix~\ref{sec:proof_prop1}).
\begin{restatable}[]{prop}{linear}
\label{prop:linear}
    Let $r := r(G_\theta(\vz_T))$ and $\bar r := \mathbb{E}[r]$, and assume $|r-\bar r|\le R$ almost surely. With $\bar r$ as the baseline,  \eqref{eq:control_exp} can be linearized as
    \begin{equation}
        \vu_{\mathrm{lin}} := \mathbb{E}[(r-\bar r) d\vw_t] /\lambda
        \label{eq:linear_estimator}
    \end{equation}
    with the error bounded by
    \begin{equation}
      \|\vu^\star - \vu_{\mathrm{lin}}\| \le \frac{e^{R/\lambda}}{2}\,\frac{R\sqrt{\mathrm{Var}(r)}}{\lambda^2} \left(1+\frac{\sqrt{\mathrm{Var}(r)}}{\lambda}\right)\sqrt{dt},
      \label{eq:prop1_bound}
    \end{equation}
    which is at most $e R\sqrt{\mathrm{Var}(r)}\sqrt{dt}/\lambda^2$ for $\lambda \ge R$.
\end{restatable}
The expectation of \eqref{eq:linear_estimator} is over the uncontrolled reference dynamics.
The resulting estimator emphasizes perturbations that produce above-average reward while suppressing those that do not.
By Proposition~\ref{prop:linear}, the linearization error is small relative to the control scale once $\lambda \gg R$, requiring only a bounded reward.
The PIC formulation further places this estimator inside a controlled SDE, whose reference dynamics we choose next.

\begin{wrapfigure}{r}{0.5\textwidth}
\vspace{-4mm}
\begin{minipage}{0.5\textwidth}
\begin{algorithm}[H]
\caption{Algorithm of ZeNO}\label{alg:method}
\begin{algorithmic}[1]
\Require Generator $G_\theta$, reward $r$, initial noise $\vz$, reference coefficient $\beta$, step size $\eta$, number of particles $N$, number of iterations $M$
\State $\vz_0 \gets \vz$
\For{$m = 0, \dots, M-1$}
  \For{$n = 1, \dots, N$} \Comment{in parallel}
    \State $\veps^{(n)} \sim \gN(0,\rmI)$
    \State $\vz_m^{(n)} \gets \sqrt{1-\beta} \vz_m + \sqrt{\beta}\veps^{(n)}$
    \State $r_m^{(n)} \gets r(G_\theta(\vz_m^{(n)}))$
  \EndFor
  \State $\bar r_m \gets \frac{1}{N}\sum_n r_m^{(n)}$
  \State $\vu_m \gets \frac{1}{N}\sum_n (r_m^{(n)} - \bar r_m)\veps^{(n)}$
  \State $\veps \sim \gN(0,\rmI)$
  \State $\vz_{m+1} \gets \sqrt{1-\beta}\vz_m + \sqrt{\beta}\veps + \eta\vu_m$
  \State $\vz_{m+1} \gets \vz_{m+1}\sqrt{d}/\|\vz_{m+1}\|$
\EndFor
\State \Return $G_\theta(\vz_M)$
\end{algorithmic}
\end{algorithm}
\end{minipage}
% \vspace{-2mm}
\end{wrapfigure}

\textbf{Choice of Reference SDE.}
The reference dynamics determine how the noise evolves in the absence of control, and thus directly affect the compatibility of the updated noise $\vz_T$ with the pre-trained generator.
Among common choices, variance-exploding (VE) SDEs lead to uncontrolled growth in noise variance, making the resulting noise incompatible with pre-trained generative models.
We therefore adopt the Ornstein--Uhlenbeck (OU) process:
\begin{equation}
    f(\vz, t):= -\frac{a_t}{2}\vz,\quad g(t):= \sqrt{a_t}
\end{equation}
which corresponds to a variance-preserving (VP) SDE with time-dependent variable $a_t$.
%
% Then, the discretized controlled OU process is given by
% \begin{equation}
% \vz \gets \sqrt{1-\beta}\vz + \sqrt{\beta}\epsilon + \eta \vu^\star.
% \end{equation}
%
% where $\beta:=a_t \Delta t$ is treated as a constant discretization coefficient.
Then, the discretized controlled OU process is given by
\begin{equation}
\vz \gets \sqrt{1-\beta} \vz + \sqrt{\beta} \veps  + \eta \vu^\star,
\label{eq:disc}
\end{equation}
where $\veps\sim\gN(0,\rmI)$ and $\beta:=a_t \Delta t$ is treated as a constant discretization coefficient. We use a small $\beta$ throughout, so that \eqref{eq:disc} is a fine discretization of the controlled SDE ($\beta = 0.01$ in all experiments).
Here, the control step size $\eta$ absorbs the coefficients arising from the Euler--Maruyama discretization and the linear approximation introduced in Proposition~\ref{prop:linear}.
The uncontrolled OU process preserves $\gN(0, \rmI)$, so perturbed noises remain compatible with the pre-trained generator. 
With control, the update becomes Langevin sampling of a reward-tilted distribution (Proposition~\ref{prop:langevin}).
For protein frames, Appendix~\ref{subsec:protein_setting} uses an analogous reference process on $SE(3)^N$.

\textbf{One-step receding-horizon.}\label{sec:horizon}
Since the reference dynamics is an SDE on the noise space, the noise injected after step $t$ accumulates in $\vz_{t+H}$, so the relative contribution of $d\vw_t$ to $\vz_{t+H}$ decreases as $H$ grows (Lemma~\ref{lem:horizon}).
This motivates the one-step receding-horizon approximation ($H=1$), where we replace $\vz_T$ with a one-step proposal $\vz' = \sqrt{1-\beta}\vz + \sqrt{\beta}\veps$ and estimate $\vu_{\mathrm{lin}}$ from $N$ proposals by $\frac{1}{N}\sum_{n}(r^{(n)}-\bar r)\,\veps^{(n)}$ up to a constant absorbed into $\eta$, with $\bar r$ the mean reward.

\textbf{Connection to Langevin dynamics.}
With $H=1$, the controlled OU update turns out to be a discretized Langevin dynamics, as formalized below.
\begin{restatable}[]{prop}{langevin}
    \label{prop:langevin}
    When $\beta \ll 1$ and the control is computed with $H=1$, the discretized controlled OU process \eqref{eq:disc} is the Langevin step
    \begin{equation}
        \vz \gets \vz + \frac{\beta}{2}\nabla_{\vz}\log p^\star_\kappa(\vz) + \sqrt{\beta}\veps,
        \label{eq:langevin}
    \end{equation}
    where $p^\star_\kappa(\vz) \propto p(\vz)\exp\big(\kappa V(\vz)\big)$ with $p(\vz) = \gN(0,\rmI)$, $V(\vz) = \log\mathbb{E}_{\veps}[\exp(r(G_\theta(\sqrt{1-\beta}\vz + \sqrt{\beta}\veps))/\lambda)]$ is the value of the one-step problem, and $\kappa = 2\eta/\beta$.
\end{restatable}
In the noise-space PIC problem, the value at the terminal time is $r(G_\theta(\vz))/\lambda$, so for a horizon long enough for the reference dynamics to forget the initial noise, the full-horizon problem approximately targets $p^\star(\vz) \propto p(\vz)\exp(r(G_\theta(\vz))/\lambda)$ at the terminal time.
Repeating the one-step receding-horizon update, which Proposition~\ref{prop:langevin} identifies as a Langevin step, instead moves the noise toward the stationary distribution $p^\star_\kappa$, which approaches $p^\star$ as $\beta \to 0$ with fixed $\kappa$, up to the temperature scaled by $\kappa$, coinciding with it for $\kappa = 1$.
Since Algorithm~\ref{alg:method} uses $\vu_{\mathrm{lin}}$ instead of $\vu^\star$, this interpretation holds for the algorithm up to the linearization error bounded in Proposition~\ref{prop:linear}.
In the full-horizon control, however, the terminal state is dominated by later perturbations (Lemma~\ref{lem:horizon}), whose effect on the reward remains as noise in the control estimate, whereas with $H=1$ the terminal state varies only through the current perturbation.

\textbf{Relation to existing methods.}
The most closely related method is ORIGEN~\citep{min2026origen}, which also runs Langevin dynamics in the initial noise space toward a reward-tilted distribution.
ORIGEN obtains the score by backpropagating through the reward--generator composition, which requires differentiable rewards and reliable gradient flow.
ZeNO targets a distribution of the same form, $p^\star_\kappa$ in Proposition~\ref{prop:langevin}, but the PIC formulation supplies the drift as an expectation of reward-weighted noise estimated from one-step reward evaluations alone, so it applies to non-differentiable and black-box rewards.

The proposed method with the OU process is described in Algorithm~\ref{alg:method} and depicted in Figure~\ref{fig:method}.
At each of $M$ iterations, $N$ one-step proposals are drawn around the current noise, their rewards are evaluated, and the resulting estimate of $\vu_{\mathrm{lin}}$ replaces $\vu^\star$ in \eqref{eq:disc} to update the noise.
The number of particles $N$ and the number of iterations $M$ thus serve as inference-time scaling knobs, and increasing either allocates more computation toward reward improvement without modifying the generator or reward model (see Section~\ref{subsec:exp_img}).

\begin{wraptable}{r}{0.44\textwidth}
\vspace{-4mm}
\centering
\caption{Comparison of empirical mode probabilities and KL divergence
from the target reward-tilted distribution.}
\label{tab:toy_tilted}
\small
\begin{tabular}{c|c|c}
\toprule
Method & Mode probability & KL $\downarrow$ \\
\midrule
Target & 0.139 / 0.673 / 0.188 & -- \\
Ours   & 0.140 / 0.700 / 0.160 & \textbf{0.0026} \\
Grad   & 0.412 / 0.384 / 0.204 & 0.2481 \\
\bottomrule
\end{tabular}
\vspace{-5mm}
\end{wraptable}
\textbf{Practical design choice.}
Although the uncontrolled OU process preserves $\gN(0, \rmI)$, the control term can shift the norm of the noise away from $\sqrt{d}$, around which the norm of a standard Gaussian vector in $\sR^d$ concentrates.
We therefore re-normalize $\vz$ to norm $\sqrt{d}$ after each update, which keeps the noise compatible with the pre-trained generator.
We treat this projection as a practical choice outside Proposition~\ref{prop:langevin}.
% Since almost all mass of $\gN(0,\rmI)$ lies near this sphere in high dimension, the projection changes the target slightly, and we treat it, together with reporting the best evaluated state, as a practical choice outside Proposition~\ref{prop:langevin}.

\section{Experiments}

\begin{figure}[t]
    \centering
    \includegraphics[width=\linewidth]{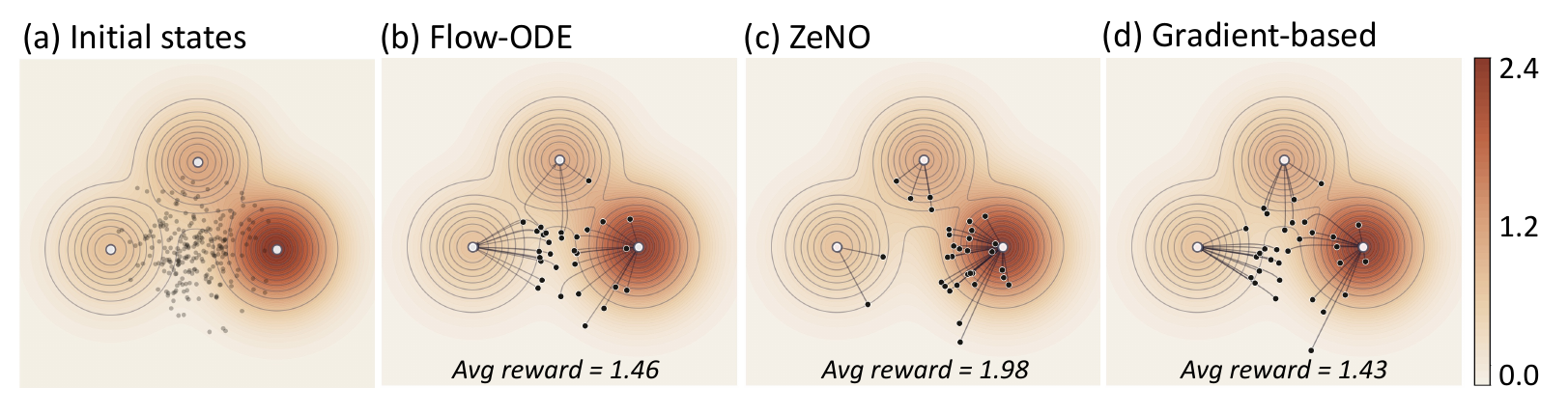}
    \vspace{-8mm}
    \caption{Toy experiment. (a) A 2D Gaussian mixture distribution with reward landscape (red: higher reward). Black dots denote initial noise samples. (b) Probability flow ODE trajectories from initial noises following the score field. (c) Trajectories from updated noises using zeroth-order path integral control (ours), consistently steering samples toward the global high-reward mode. (d) Trajectories from updated noises using first-order gradient-based updates, which tend to be guided toward nearby modes due to their local nature, yielding suboptimal reward.}
    \label{fig:toy}
    \vspace{-5mm}
\end{figure}

\subsection{Toy Experiment}
We begin with a two-dimensional toy experiment to illustrate the core behavior of the proposed method.
Specifically, we consider a Gaussian mixture model as the target data distribution and define a deterministic generative process $G_\phi$ via an ODE whose velocity field follows the score of this distribution.
Each mode is assigned a distinct reward value, inducing a non-uniform reward landscape over the data space, and the goal is to steer generated samples toward higher-reward modes. Figure~\ref{fig:toy}a illustrates the reward landscape, where brown colors indicate higher reward.
Without reward guidance, $G_\phi$ causes each initial noise to move toward its nearest mode (Figure~\ref{fig:toy}b).
When updated using ZeNO, the noise distribution shifts toward high-reward regions without collapsing to a single mode, reflecting the Langevin sampling implicitly induced by the control (Figure~\ref{fig:toy}c).
In contrast, first-order gradient updates yield a lower average reward than even the uncontrolled baseline (Figure~\ref{fig:toy}d)\footnote{Step sizes are set so that the norms of noise updates are matched to those of ZeNO; the gradient-based reward improves only marginally with larger step sizes, confirming the issue is directional.}, as local gradient information can guide noise toward nearby but suboptimal modes in flat or ambiguous regions of the reward landscape.
Table~\ref{tab:toy_tilted} further confirms these observations: ZeNO closely matches the target tilted distribution, whereas gradient-based updates deviate significantly, concentrating mass on a subset of modes.
While ZeNO is primarily motivated by non-differentiable or black-box reward settings, this result suggests an additional advantage over gradient-based methods in terms of robustness to local sensitivity issues.

\begin{figure}[t]
    \centering
    \includegraphics[width=\linewidth]{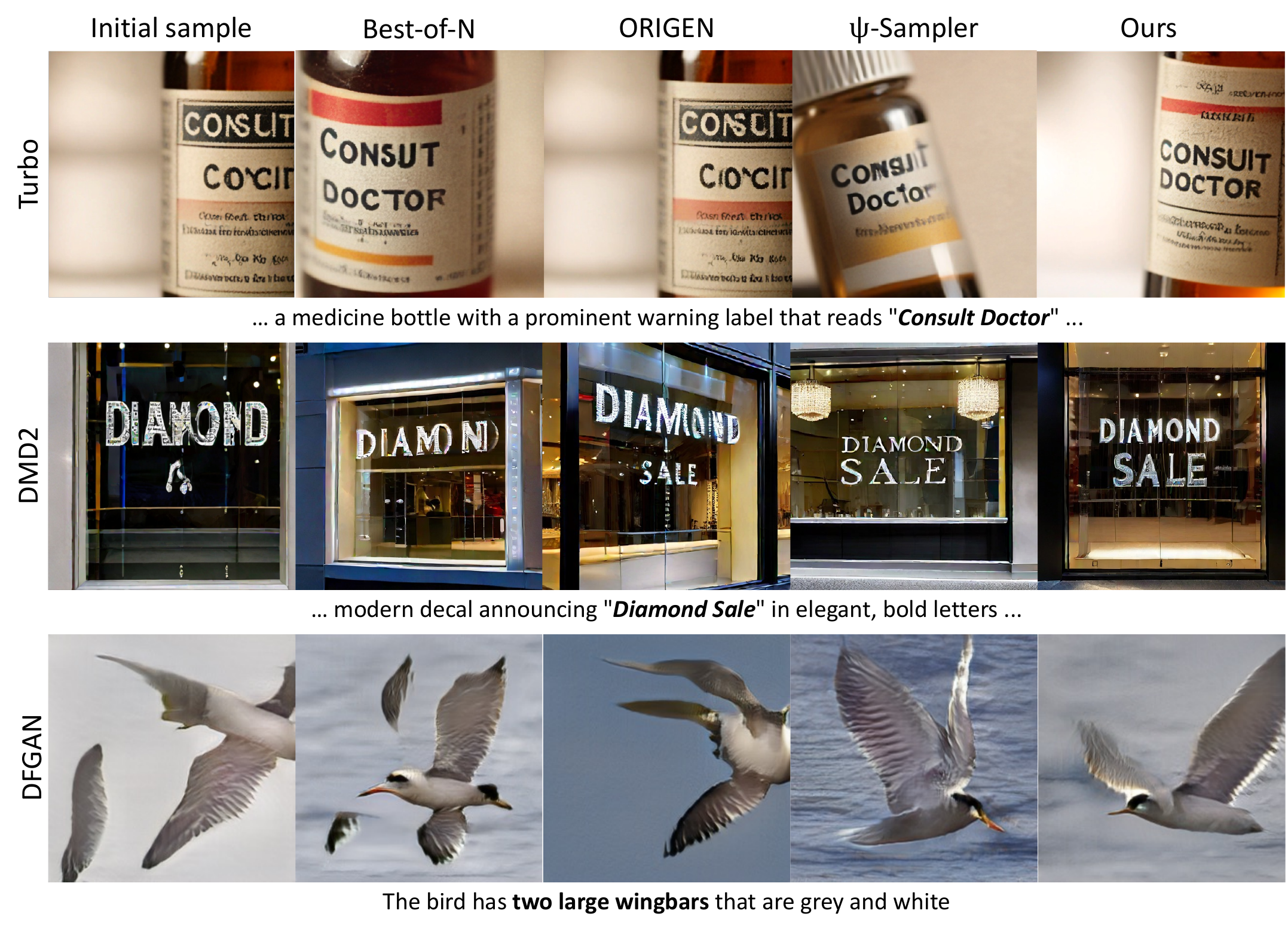}
    \vspace{-4mm}
    \caption{Qualitative comparison of ZeNO against baselines on SDXL-Turbo, DMD2, and DFGAN with PickScore (rows 1, 2) and aesthetic score (row 3) as reward functions.}
    \label{fig:qual}
    \vspace{-3mm}
\end{figure}

\subsection{Image Generation Experiments}
\label{subsec:exp_img}

\textbf{Generators and Rewards.}
We consider four generators covering distilled diffusion models, consistency models, and GANs.
DMD2 and SDXL-Turbo are distilled diffusion models based on SDXL, via distribution matching distillation~\citep{yin2024improved} and adversarial diffusion distillation~\citep{sauer2024adversarial, sauer2024fast}, respectively; we use their one-step variants.
LCM is a consistency model based on SDXL, enabling one-step generation via consistency distillation~\citep{song2023consistency, luo2023latent}.
DFGAN~\citep{tao2022df} is a GAN-based generator included to demonstrate that ZeNO is applicable beyond diffusion-based architectures.
For reward functions, we use aesthetic score~\citep{schuhmann2022aesthetic} and PickScore~\citep{kirstain2023pickapic} as differentiable rewards, and GenEval~\citep{ghosh2023geneval} as a non-differentiable reward based on compositional text-image alignment.
For DMD2, SDXL-Turbo, and LCM, we evaluate on aesthetic score using simple-animal prompts following~\citet{black2024training, clark2024directly} and on PickScore~\citep{kirstain2023pickapic} using 200 evaluation prompts.
For DFGAN, we use 500 CUB~\citep{wah_branson_welinder_perona_belongie_2011} evaluation prompts.

\textbf{Baselines.}
We compare ZeNO against the following baselines.
\textbf{Best-of-$\mathbf{N}$} independently samples $N$ noise vectors and selects the one yielding the highest reward, serving as the simplest inference-time scaling baseline applicable to any reward function.
\textbf{ORIGEN}~\citep{min2026origen} formulates noise optimization as Langevin dynamics targeting the reward-tilted distribution, representing the strongest gradient-based baseline in the noise space optimization setting.
We note that ReNO~\citep{eyring2024reno} similarly applies gradient ascent in the initial noise space; as ORIGEN subsumes this approach within a more principled formulation, we use ORIGEN as the representative gradient-based baseline.
Both methods require differentiable rewards and reliable gradient flow through the generator.
\textbf{$\Psi$-Sampler}~\citep{yoon2026psisampler} applies pCNL-based Langevin dynamics in the initial noise space to sample from the reward-aware posterior.
We adapt $\Psi$-Sampler by applying the pCNL update directly in the initial noise space with the deterministic generators.
For gradient-based methods (ORIGEN and $\Psi$-Sampler), we match the number of generator calls under the same budget; note that these methods additionally require backpropagation through the generator and reward, incurring additional computational overhead.

\textbf{Hyperparameters.}
For all experiments, we use $\beta=0.01$, and control step size $\eta=1.5$. For the main experiments, we run each algorithm with $N=16$ and $M=200$ to match the compute budget, and report the best reward among noise states. See Appendix~\ref{sec:beta_ablation} for an ablation on $\beta$.

\begin{figure}[t]
\centering
% Top row: two tables side by side
\begin{minipage}[t]{0.58\textwidth}
\vspace{0pt}
\centering
\captionof{table}{
Results on differentiable rewards. We report aesthetic score and PickScore across four generators.
\textbf{Bold} indicates the best, and \underline{underline} indicates the second-best result. $^\dagger$ Evaluated on 500 CUB prompts~\citep{wah_branson_welinder_perona_belongie_2011}.
}
\label{tab:differentiable_rewards}
\resizebox{\linewidth}{!}{
\begin{tabular}{lcc|cc|cc|cc}
\toprule
& \multicolumn{2}{c}{DMD2}
& \multicolumn{2}{c}{SDXL-Turbo}
& \multicolumn{2}{c}{LCM}
& \multicolumn{2}{c}{DFGAN$^\dagger$} \\
\midrule
Method
& Aes. & Pick.
& Aes. & Pick.
& Aes. & Pick.
& Aes. & Pick. \\
\midrule
Uncontrolled
& 4.42 & 20.49 & 5.43 & 21.65 & 4.59 & 18.53 & 4.26 & 19.48 \\
Best-of-$N$
& 4.73 & 22.29 & 4.65 & 23.08 & 4.80 & 20.65 & \textbf{5.65} & \underline{21.33} \\
ORIGEN
& 5.97 & 21.83 & \underline{5.93} & 21.50 & 4.45 & 18.83 & 5.34 & 20.98 \\
$\Psi$-Sampler
& \underline{6.22} & \underline{22.56} & \textbf{6.00} & \underline{23.35} & \underline{5.90} & \underline{21.11} & 5.20 & 20.81 \\
ZeNO (ours)
& \textbf{6.31} & \textbf{22.68}
& \textbf{6.00} & \textbf{23.43}
& \textbf{6.15} & \textbf{21.53}
& \underline{5.63} & \textbf{21.34} \\
\bottomrule
\end{tabular}
}
\end{minipage}
\hfill
\begin{minipage}[t]{0.40\textwidth}
\vspace{0pt}
\centering
\captionof{table}{
Results on non-differentiable rewards. We report GenEval scores for count and position, where gradient-based methods are inapplicable.
}
\label{tab:nondifferentiable_rewards}
\resizebox{\linewidth}{!}{
\begin{tabular}{lcc|cc}
\toprule
& \multicolumn{2}{c}{DMD2}
& \multicolumn{2}{c}{SDXL-Turbo} \\
\midrule
Method
& Count & Pos.
& Count & Pos. \\
\midrule
Uncontrolled
& 0.163 & 0.090 & 0.513 & 0.250 \\
Best-of-$N$
& \underline{0.838} & \textbf{0.840}
& \underline{0.925} & \textbf{0.900} \\
ZeNO (ours)
& \textbf{0.950} & \underline{0.760}
& \textbf{0.963} & \underline{0.680} \\
\bottomrule
\end{tabular}
}
\end{minipage}

\vspace{3mm}

% Bottom row: figures
\begin{minipage}[t]{0.61\textwidth}
\vspace{0pt}
\centering
\includegraphics[width=\linewidth]{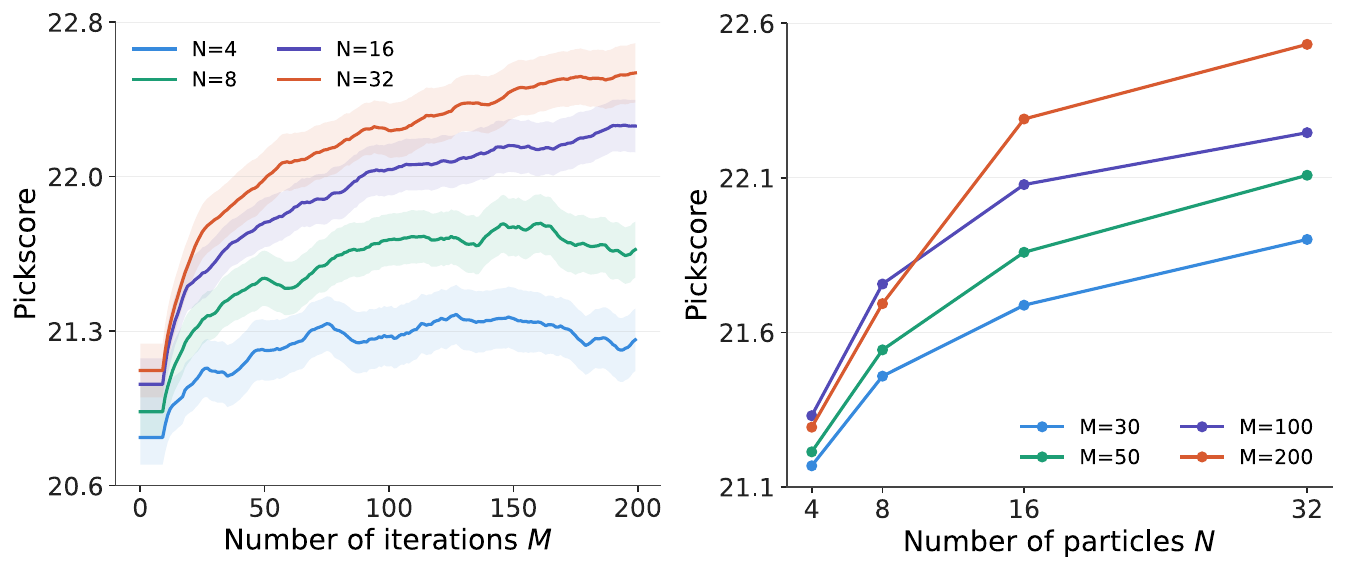}
\captionof{figure}{Inference-time scaling ablation for $N$ and $M$.}
\label{fig:scaling}
\end{minipage}
\hfill
\begin{minipage}[t]{0.36\textwidth}
\vspace{-7mm}
\centering
\includegraphics[width=\linewidth]{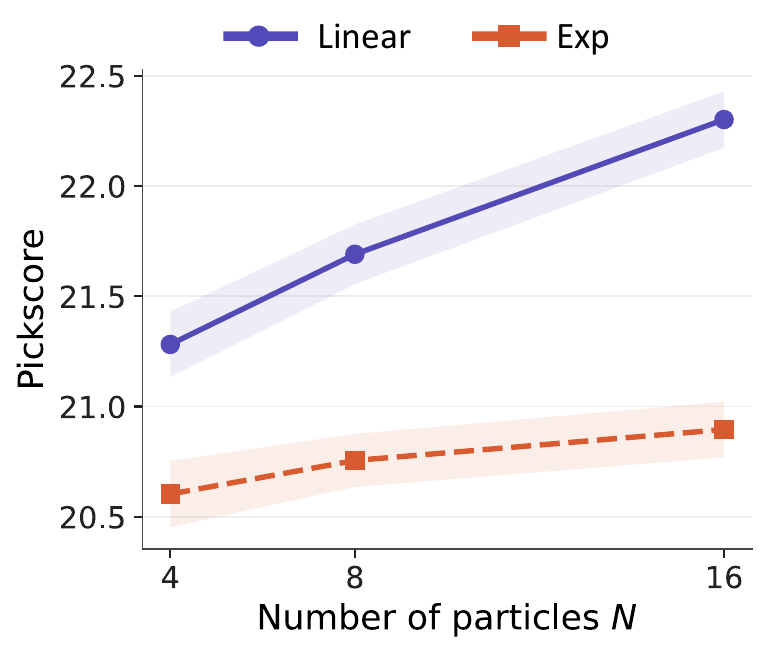}
\captionof{figure}{Estimator comparison.}
\label{fig:estimator_comparison}
\end{minipage}
\vspace{-5mm}
\end{figure}

\textbf{Comparison with Baselines.}
Table~\ref{tab:differentiable_rewards} reports results on aesthetic score and PickScore across all generators.
ZeNO consistently improves reward across all generator and reward combinations.
Compared to ORIGEN and $\Psi$-Sampler, ZeNO achieves competitive or superior performance despite requiring no gradients through the generator or reward model, demonstrating that zeroth-order noise optimization is a viable alternative to gradient-based methods.
In particular, ORIGEN degrades below the uncontrolled baseline on LCM, consistent with our toy experiment findings that local gradient information can be unreliable, and further suggesting that this issue is exacerbated when backpropagating through complex generator-reward compositions.
Notably, DFGAN results confirm that ZeNO generalizes beyond diffusion-based architectures, supporting its generator-agnostic applicability.

Table~\ref{tab:nondifferentiable_rewards} reports results on GenEval, where the reward is computed by an object detection-based evaluation pipeline and is non-differentiable.
In this setting, gradient-based methods including ORIGEN and $\Psi$-Sampler are inapplicable, and we compare ZeNO against Best-of-$N$ under the same total generator call budget.
We evaluate on counting and position subsets of GenEval, where uncontrolled generation consistently fails, providing meaningful room for reward-guided improvement.
ZeNO consistently outperforms Best-of-$N$ on counting tasks across all generators, demonstrating that iterative noise optimization extracts more reward improvement from the same computational budget than independent sampling.
On position tasks, however, Best-of-$N$ is better in some settings. We attribute this to the initial point dependency inherent in our one-step VP SDE formulation: since $\vz_{t+1}$ is sampled in a single step from $\vz_t$, the optimization heavily depends on the initial noise state, which is particularly limiting for position tasks where small spatial differences lead to large reward changes.
Nevertheless, ZeNO outperforms Best-of-$N$ across other rewards and counting tasks, demonstrating broad effectiveness as an inference-time optimization framework.

\textbf{Inference-Time Scaling Analysis.}
Figure~\ref{fig:scaling} shows PickScore of DMD2 as a function of the number of update iterations $M$ (left) and the number of particles $N$ (right), confirming that both serve as effective inference-time scaling knobs.
Both sweeps confirm steady reward improvement: larger $N$ consistently achieves higher reward at any given $M$, and larger $M$ yields higher reward at any given $N$.
These results demonstrate that ZeNO enables flexible inference-time scaling, allowing a trade-off between the number of particles and iterations under a given compute budget.

\textbf{Estimator Comparison.}
Figure~\ref{fig:estimator_comparison} compares the linearized estimator adopted in ZeNO (Proposition~\ref{prop:linear}) with the exponential weighting in \eqref{eq:control_exp} across the number of particles $N$.
Both estimators improve as $N$ increases, consistent with the variance reduction effect of using more particles.
However, the linearized estimator substantially outperforms the exponential weighting, with the gap widening as $N$ increases, consistent with the self-normalization of the latter discussed in Section~\ref{sec:zeno}.

\subsection{Protein Structure Generation Experiments}

\begin{wrapfigure}{r}{0.35\textwidth}
    \vspace{-7mm}
    \centering
    \includegraphics[width=0.35\textwidth]{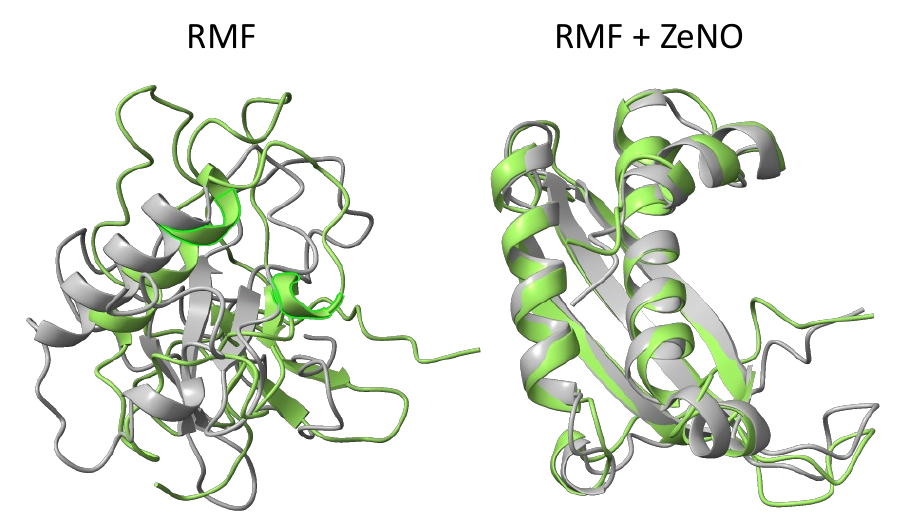}
    \caption{Structural overlay of the generated protein backbone and the ESMFold-refolded structure.}
    \label{fig:rmf_structure}
    \vspace{-4mm}
\end{wrapfigure}

Most of the practically relevant objectives in protein design are non-differentiable or only available through external black-box pipelines. To demonstrate the applicability of ZeNO, we utilize the one-step Riemannian MeanFlow (RMF)~\citep{woo2026riemannianmeanflow} protein backbone generator. Following the representation used in frame-based protein generative models such as FrameDiff and FrameFlow, RMF represents protein backbones as residue-wise $\mathrm{SE}(3)$ frames and generates structures by updating these frames through a learned flow map. We use scRMSD as a black-box reward, measured by a self-consistency protocol based on inverse folding (ProteinMPNN~\citep{proteinMPNN}) and refolding (ESMFold~\citep{ESMFold}). A backbone is considered designable when scRMSD $<2\text{\AA}$. Further details are provided in Appendix~\ref{subsec:protein_setting}.
Table~\ref{tab:rmf_comparison} and Figure~\ref{fig:rmf_structure} show that applying ZeNO to RMF substantially improves backbone designability. For both RMF/S and RMF/M model, ZeNO increases the fraction of designable samples and reduces the average scRMSD. This results indicate that test-time optimization of the initial frame noise can effectively improve sample quality without modifying the generative model. More notably, with only a single flow evaluation, RMF+ZeNO achieves substantially higher designability than FrameFlow and FrameDiff with 5 NFEs, suggesting that ZeNO can partially compensate for reduced sampling steps by exploiting reward-guided optimization in the initial noise space. At the same time, the number of structural clusters increases for both models, while the pairwise scTM changes only marginally. This shows that the improvement in designability does not come from collapsing to a narrow set of similar structures. Figure~\ref{fig:rmf_compare} further supports this interpretation. As the ZeNO iterations progress, the scRMSD steadily decreases, demonstrating that the control signal consistently guides the initial frames toward more designable protein backbones.

\begin{table}[t]
\centering
\begin{minipage}[t]{0.62\textwidth}
\vspace{0pt}
\centering
\caption{Comparison of designability and diversity across different NFE regimes. With ZeNO, RMF achieves stronger designability even at a single NFE than baselines.}
\label{tab:rmf_comparison}
\resizebox{\linewidth}{!}{
\begin{tabular}{@{}lccccc@{}}
\toprule
\multirow{2}{*}{Model}
& \multirow{2}{*}{NFE}
& \multicolumn{2}{c}{Designability}
& \multicolumn{2}{c}{Diversity} \\
\cmidrule(lr){3-4} \cmidrule(lr){5-6}
&
& \makecell{$<2\text{\AA}$ \\ $(\uparrow)$}
& \makecell{scRMSD \\ $(\downarrow)$}
& \makecell{Max. \\ Cluster $(\uparrow)$}
& \makecell{Pairwise \\ scTM $(\downarrow)$} \\
\midrule
FrameFlow & 5 & 0.04 & 6.53 & 0.68 & 0.22 \\
FrameDiff & 5 & 0.09 & 6.19 & 0.54 & 0.24 \\
FrameFlow & 2 & 0.00 & N/A & N/A & N/A \\
% RMF       & 5 & 0.82 & 1.54 & 0.54 & 0.27 \\
\midrule
RMF/S     & 1 & 0.03 & 7.93 & 0.84 & 0.33 \\
\quad + \textbf{ZeNO} & 1 & \textbf{0.15} & \textbf{5.67} & \textbf{0.88}& \underline{0.34} \\
RMF/M     & 1 & 0.14 & 4.78 & 0.60 & 0.33 \\
\quad + \textbf{ZeNO} & 1 & \textbf{0.32} & \textbf{3.54} & \textbf{0.85}& \underline{0.35} \\
\bottomrule
\end{tabular}
}
\end{minipage}
\hfill
\begin{minipage}[t]{0.36\textwidth}
\vspace{0pt}
\centering
\includegraphics[width=\linewidth]{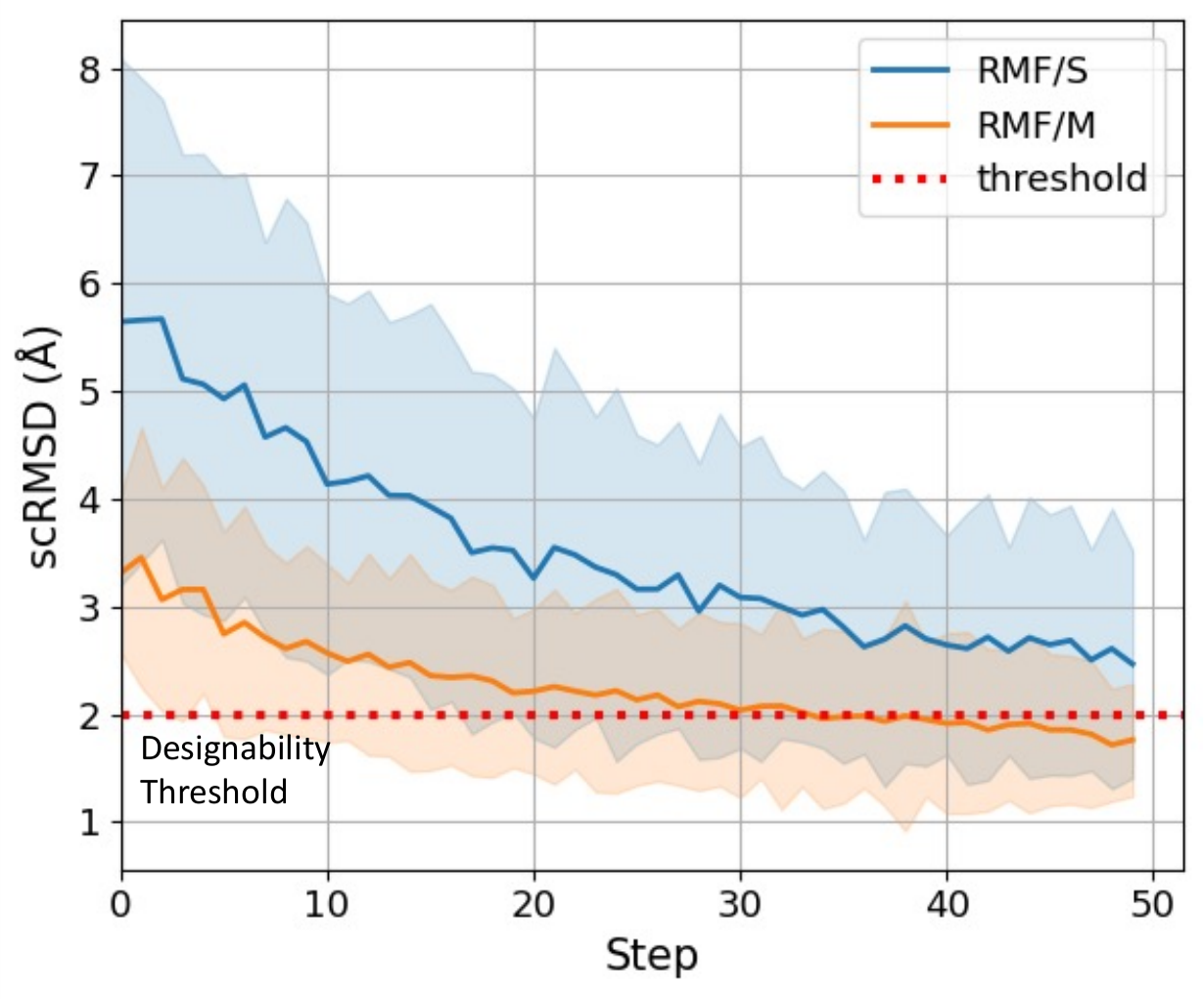}
\captionof{figure}{Evolution of designability with ZeNO on RMF-generated protein backbones.}
\label{fig:rmf_compare}
\end{minipage}
\vspace{-5mm}
\end{table}

\section{Conclusion}
We proposed ZeNO, a zeroth-order noise optimization framework that aligns deterministic generators with reward functions without requiring gradients, formulating noise updates as a path-integral control problem to yield a Langevin-like update rule targeting the reward-tilted distribution while remaining compatible with pre-trained generators via Ornstein--Uhlenbeck reference dynamics.
Empirically, ZeNO enables effective inference-time scaling and demonstrates strong performance across a range of generators and reward functions, including settings where gradient-based methods are fundamentally inapplicable.
A current limitation is that the zeroth-order estimator requires multiple reward evaluations per update step, which can be computationally demanding when the generator or reward model is expensive to evaluate.

% ============================================================
% ICLR statements (do not count toward the page limit)
% ============================================================
\subsection*{AI use statement}
% REQUIRED by ICLR 2027. Please fill in according to the ICLR 2027 AI Policy for Authors.
In this work, we used generative AI tools for sentence-level editing of texts written by authors to improve readability. We have not used generative AI tools for other aspects, such as generating synthetic data, developing theoretical models, formulating mathematical claims or proposing research hypothesis. 
We have reviewed all AI-assisted work.  We take responsibility for the final content of this work,
including text, claims or artifacts produced with the aid of generative AI.

\subsection*{Ethics statement}
This work proposes an inference-time optimization method for aligning pre-trained generative models with reward functions. As with other reward alignment methods, ZeNO may inherit biases of the underlying generators and reward models, and could in principle be used to steer generators toward undesirable content if paired with a harmful reward. Our experiments use only publicly available models, prompts, and datasets, and do not involve human subjects.

\subsection*{Reproducibility statement}
The full algorithm is given in Algorithm~\ref{alg:method}, and all hyperparameters for image generation are reported in Section~\ref{subsec:exp_img} and Appendix~\ref{subsec:image_setting}. Complete proofs of Propositions~\ref{prop:linear}, \ref{prop:langevin} and Lemma \ref{lem:horizon} are provided in Appendices~\ref{sec:proof_prop1}-\ref{sec:proof_prop3}. The protein reference dynamics, control estimation, and reward are described in detail in Appendix~\ref{subsec:protein_setting}. Code will be made publicly available upon publication.

\bibliography{main}
\bibliographystyle{iclr2027_conference}

%%%%%%%%%%%%%%%%%%%%%%%%%%%%%%%%%%%%%%%%%%%%%%%%%%%%%%%%%%%%

\newpage
\appendix

\section{Path-Integral Control Formulation}
\label{sec:pic}

We briefly derive the optimal control form used in the main text, following~\citet{pathintegral}.
Consider the controlled SDE
\begin{equation}
    d\vz = f(\vz,t)dt + g(t)(\vu(\vz,t)dt+d\vw_t),
\end{equation}
with terminal cost
\begin{equation}
    m(\vz_T)=-r(G_\theta(\vz_T)).
\end{equation}
The path-integral control problem seeks the optimal control that minimizes
\begin{equation}
    J(\vu)=\mathbb{E}\left[m(\vz_T)+ \lambda \int_0^T \frac{1}{2} \|\vu(\vz_t,t)\|^2 dt \right],
\end{equation}
where the quadratic control cost regularizes the deviation from the uncontrolled reference dynamics.

Define the desirability function
\begin{equation}
\Psi(\vz,t)=\mathbb{E}\left[\exp(r(G_\theta(\vz_T))/\lambda)\mid\vz_t=\vz\right],
\end{equation}
and the value function
\begin{equation}
    V(\vz,t)=\log \Psi(\vz,t),
\end{equation}
consistent with \eqref{eq:tilted}, so that the optimal cost-to-go is $-\lambda V$.
By the Feynman-Kac theorem, the value function satisfies the corresponding Hamilton-Jacobi-Bellman equation after the logarithmic transformation. The optimal control is given by the score of the desirability function:
\begin{equation}
    \vu^\star(\vz,t) = g(t)\nabla_\vz V(\vz,t) = g(t)\nabla_\vz \log \Psi(\vz,t).
\end{equation}
In our implementation, the scalar factor $g(t)$ is absorbed into the update step size. Thus, we use the following normalized control:
\begin{equation}
    \vu^\star(\vz,t) = \nabla_\vz \log \mathbb{E} \left[ \exp(r(G_\theta(\vz_T))/\lambda) \mid \vz_t=\vz \right].
    \label{eq:pic_control_gradient}
\end{equation}

Equivalently, the control can be written in path-integral form under the uncontrolled reference dynamics, where, as in \eqref{eq:control_exp}, $\vu^\star$ denotes the control increment and the time and diffusion coefficients are absorbed into the step size. Let $d\vw_t$ denote the Brownian increment at time $t$. Then,
\begin{equation}
\vu^\star(\vz,t):=\frac{\mathbb{E}_{\vz_{t:T}}\left[\exp(r(G_\theta(\vz_T))/\lambda) d\vw_t \right]}{\mathbb{E}_{\vz_{t:T}}\left[\exp(r(G_\theta(\vz_T))/\lambda)\right]}.
    \label{eq:pic_control_exp}
\end{equation}
This expression shows that the optimal control can be estimated using only zeroth-order reward evaluations along uncontrolled trajectories.

\section{Proof of Proposition 1}
\label{sec:proof_prop1}

\linear*
\begin{proof}
Let $\delta := (r-\bar r)/\lambda$.
Then, $|\delta| \le R/\lambda$, $\mathbb{E}[\delta] = 0$, and $\vu_{\mathrm{lin}} = \mathbb{E}[\delta\, d\vw_t]$.
We apply the Taylor expansion with the remainder as $e^{\delta} = 1 + \delta + \rho$, where $\rho = e^{\zeta}\delta^2/2 \ge 0$ for some $\zeta$ between 0 and $\delta$.
Also, since $|\zeta| \le |\delta| \le R/\lambda$, we have $e^{\zeta} \le e^{R/\lambda}$, which gives
\begin{equation}
    0 \le \rho \le \tfrac12 e^{R/\lambda}\delta^2.
\end{equation}

Since $e^{r/\lambda} = e^{\bar r/\lambda}e^{\delta}$, the factor $e^{\bar r/\lambda}$ cancels in \eqref{eq:control_exp}.
Then, the denominator of the optimal control is
\begin{equation}
    \mathbb{E}[e^{\delta}] = 1 + \mathbb{E}[\rho],
\end{equation}
where the equality follows from $\mathbb{E}[\delta] = 0$, and the numerator is
\begin{equation}
    \mathbb{E}[e^{\delta} d\vw_t] = \vu_{\mathrm{lin}} + \mathbb{E}[\rho\, d\vw_t],
\end{equation}
where the equality holds as $\mathbb{E}[d\vw_t] = 0$.
By combining them, we get
\begin{equation}
    \vu^\star - \vu_{\mathrm{lin}} = \frac{\mathbb{E}[\rho\, d\vw_t] - \mathbb{E}[\rho]\,\vu_{\mathrm{lin}}}{1 + \mathbb{E}[\rho]},
\end{equation}
which gives $\vu^\star = \vu_{\mathrm{lin}}$ when $\rho = 0$.
Since $\rho \ge 0$, the denominator is at least $1$, and thus
\begin{equation}
    \|\vu^\star - \vu_{\mathrm{lin}}\| \le \|\mathbb{E}[\rho\, d\vw_t]\| + \mathbb{E}[\rho]\,\|\vu_{\mathrm{lin}}\|.
\end{equation}
By the Cauchy--Schwarz inequality, using $\vv^\top d\vw_t \sim \gN(0, dt)$ for any unit vector $\vv$, we have
\begin{equation}
    \|\mathbb{E}[\rho\, d\vw_t]\| \le \sqrt{\mathbb{E}[\rho^2]}\,\sqrt{dt},
    \qquad
    \|\vu_{\mathrm{lin}}\| \le \frac{\sqrt{\mathrm{Var}(r)}\,\sqrt{dt}}{\lambda}.
\end{equation}
Also, since $|r-\bar r| \le R$, the bound on $\rho$ gives
\begin{equation}
    \mathbb{E}[\rho^2] \le \frac{e^{2R/\lambda} R^2\,\mathrm{Var}(r)}{4\lambda^4},
    \qquad
    \mathbb{E}[\rho] \le \frac{e^{R/\lambda}\,\mathrm{Var}(r)}{2\lambda^2}.
\end{equation}
By combining them, we get
\begin{equation}
    \|\vu^\star - \vu_{\mathrm{lin}}\| \le \frac{e^{R/\lambda}}{2}\left(\frac{R\sqrt{\mathrm{Var}(r)}}{\lambda^2} + \frac{\mathrm{Var}(r)^{3/2}}{\lambda^3}\right)\sqrt{dt}.
\end{equation}
Since $\sqrt{\mathrm{Var}(r)} \le R$, we have 
\begin{equation}
    \frac{\mathrm{Var}(r)^{3/2}}{\lambda^3} \le R \frac{\mathrm{Var}(r)}{\lambda^3},
\end{equation}
which gives \eqref{eq:prop1_bound}.
When $\lambda \ge R$, $e^{R/\lambda} \le e$ and $\sqrt{\mathrm{Var}(r)} \le R \le \lambda$, which gives $e\,R\sqrt{\mathrm{Var}(r)}\sqrt{dt}/\lambda^2$.
\end{proof}

\paragraph{Interpretation.}
Since $\lambda\log\mathbb{E}[e^{r/\lambda}] = \mathbb{E}[r] + \mathrm{Var}(r)/(2\lambda) + O(\lambda^{-2})$, the linearized control keeps only the average reward, i.e., the risk-neutral part of the risk-sensitive PIC objective.
Both $\vu^\star$ and $\vu_{\mathrm{lin}}$ scale as $1/\lambda$, which is absorbed into $\eta$, while their difference scales as $1/\lambda^2$, so the linearization error becomes small relative to the scale of the control as $\lambda$ grows.
Moreover, $\lambda$ does not appear in Algorithm~\ref{alg:method}, and to first order the target of Proposition~\ref{prop:langevin} depends on $\lambda$ and $\kappa$ only through $\lambda/\kappa$, so $\lambda$ can be taken large relative to $R$ without changing the target.

\section{Proof of Proposition 2}
\label{sec:proof_rem1}

\langevin*
\begin{proof}
A Langevin step with step size $s$ toward a distribution $\pi$ reads $\vz \gets \vz + s\,\nabla_{\vz}\log\pi(\vz) + \sqrt{2s}\,\veps$.
We show that \eqref{eq:disc} takes this form with $s = \beta/2$ and $\pi = p^\star_\kappa$.

When $\beta \ll 1$, we use $\sqrt{1-\beta} = 1 - \beta/2 + O(\beta^2)$, which gives, to first order in $\beta$,
\begin{equation}
    \vz \gets \vz + \frac{\beta}{2}\left(-\vz + \kappa\vu^\star\right) + \sqrt{\beta}\,\veps, \qquad \kappa = 2\eta/\beta,
\end{equation}
where we keep $\kappa$ fixed.
The noise term already matches, since $\sqrt{2\cdot\beta/2} = \sqrt{\beta}$, so it remains to show that the drift $-\vz + \kappa\vu^\star$ equals $\nabla_{\vz}\log p^\star_\kappa(\vz)$.

For the first term, since $p(\vz) = \gN(0,\rmI)$, we have $-\vz = \nabla_{\vz}\log p(\vz)$.
For the second term, the control with $H=1$ is the optimal control of the problem whose terminal time is one step ahead.
Applying \eqref{eq:pic_control_gradient} to this problem, i.e., with $\vz_T$ replaced by the one-step proposal $\vz' = \sqrt{1-\beta}\vz + \sqrt{\beta}\veps$, gives
\begin{equation}
    \vu^\star = \nabla_{\vz}\log\mathbb{E}_{\veps}\left[\exp\left(r(G_\theta(\vz'))/\lambda\right)\right] = \nabla_{\vz}V(\vz).
\end{equation}
Then, the drift becomes
\begin{equation}
    -\vz + \kappa\nabla_{\vz}V(\vz) = \nabla_{\vz}\log\big(p(\vz)\exp(\kappa V(\vz))\big) = \nabla_{\vz}\log p^\star_\kappa(\vz),
\end{equation}
which proves the claim. Since $V$ is the same at every iteration, iterating this step gives a Langevin chain whose stationary distribution approximates $p^\star_\kappa$ for small $\beta$.
\end{proof}

\paragraph{Optimal noise distribution.}
For $p^\star(\vz) \propto p(\vz)\exp(r(G_\theta(\vz))/\lambda)$ with normalizing constant $Z$, we have, for any distribution $q$,
\begin{equation}
    \mathbb{E}_{\vz\sim q}[r(G_\theta(\vz))] - \lambda\,\mathrm{KL}(q\,\|\,p) = \lambda\log Z - \lambda\,\mathrm{KL}(q\,\|\,p^\star).
\end{equation}
Thus, $p^\star$ is the unique maximizer of the KL-regularized objective.

\paragraph{Relation to the full-horizon problem.}
Given the initial noise $\vz_0$, the optimal path measure is the reference path measure from $\vz_0$ tilted by $\exp(r(G_\theta(\vz_T))/\lambda)$.
As this weight depends only on the endpoint, the terminal marginal is proportional to $p^{\text{ref}}_T(\vz\mid\vz_0)\exp(r(G_\theta(\vz))/\lambda)$.
For a horizon long enough for the reference dynamics to forget $\vz_0$, $p^{\text{ref}}_T(\cdot\mid\vz_0)$ approaches $p$, and thus the terminal marginal approximately becomes $p^\star$.

\paragraph{Limit of the one-step target.}
The one-step control does not plan the remaining trajectory, and instead moves the noise toward the stationary distribution $p^\star_\kappa$ of the iterated Langevin step.
As $\beta \to 0$ with fixed $\kappa$, the one-step proposal $\vz'$ concentrates at $\vz$, so $V(\vz)$ approaches $r(G_\theta(\vz))/\lambda$, and $p^\star_\kappa$ approaches $p^\star$ up to the temperature scaled by $\kappa$, coinciding with it for $\kappa = 1$.

\section{Proof of Lemma 1}
\label{sec:proof_prop3}

\begin{lemma}\label{lem:horizon}
Under the uncontrolled process \eqref{eq:disc}, the fraction of the variance of $\vz_{t+H}$ given $\vz_t$ that is due to $\veps_t$ is
\begin{equation}
    s_H = \frac{\beta(1-\beta)^{H-1}}{1-(1-\beta)^{H}},
\end{equation}
which equals $1$ at $H=1$ and decreases monotonically in $H$.
\end{lemma}

\begin{proof}
Setting $\vu^\star = 0$ in \eqref{eq:disc}, the uncontrolled process gives
\begin{equation}
    \vz_{t+1} = \sqrt{1-\beta}\,\vz_t + \sqrt{\beta}\,\veps_t.
\end{equation}
Applying this recursively for $H$ steps,
\begin{align}
    \vz_{t+H}
    &= \sqrt{1-\beta}\,\vz_{t+H-1} + \sqrt{\beta}\,\veps_{t+H-1} \\
    &= (1-\beta)\,\vz_{t+H-2} + \sqrt{\beta}(1-\beta)^{1/2}\veps_{t+H-2} + \sqrt{\beta}\,\veps_{t+H-1} \\
    &= \cdots \\
    &= (1-\beta)^{H/2}\vz_t + \sqrt{\beta}(1-\beta)^{(H-1)/2}\veps_t + \sqrt{\beta}\sum_{k=0}^{H-2}(1-\beta)^{k/2}\veps_{t+H-1-k}.
\end{align}
Given $\vz_t$, the noise terms are independent, so the variance of $\vz_{t+H}$ per coordinate is the sum of the squared coefficients,
\begin{equation}
    \underbrace{\beta(1-\beta)^{H-1}}_{\text{from }\veps_t} + \beta\sum_{k=0}^{H-2}(1-\beta)^{k}
    = \beta\sum_{k=0}^{H-1}(1-\beta)^{k}
    = 1-(1-\beta)^{H},
\end{equation}
and the fraction due to $\veps_t$ is $s_H$.
Let $q := (1-\beta)^{H-1}\in(0,1]$. Then, $s_H = \beta q/(1-(1-\beta)q)$ increases in $q$, and $q$ decreases in $H$, so $s_H$ decreases in $H$ with $s_1 = 1$.
\end{proof}

\section{Ablation Study}
\label{sec:additional_exp}

\subsection{$\beta$ ablation}
\label{sec:beta_ablation}

We conduct an ablation study on the SDE variance parameter $\beta$ with SDXL-Turbo. As shown in Appendix~\ref{sec:proof_rem1}, the Langevin-like sampling interpretation targeting the tilted prior distribution holds in the small-variance regime, i.e., when $\beta \ll 1$. As shown in Table~\ref{tab:vp_beta_ablation}, $\beta = 0.01$ achieves the highest aesthetic score by a large margin, whereas all larger values of $\beta$ yield nearly identical and substantially lower scores. 

\begin{table}[h]
\centering
\caption{Ablation study on the effect of $\beta$ on aesthetic score.}
\label{tab:vp_beta_ablation}
\resizebox{0.35\textwidth}{!}{
\begin{tabular}{cc}
\toprule
$\beta$ & Aesthetic Score $(\uparrow)$ \\
\midrule
0.01 & \textbf{6.000} \\
0.05 & 4.585 \\
0.10 & 4.582 \\
0.30 & 4.586 \\
\bottomrule
\end{tabular}
}
\end{table}

\subsection{Perturbation Ablation}
\label{app:perturbation_ablation}

The OU reference process is motivated by preserving the Gaussian prior, while the practical algorithm combines one-step discrete perturbations with an explicit norm projection (Algorithm~\ref{alg:method}). To examine whether the proposed control update contributes beyond simpler noise-space search schemes, we compare ZeNO with the following baselines on PickScore with SDXL-Turbo. All variants use the same evaluation budget and hyperparameters as ZeNO ($\beta = 0.01$, $\eta = 1.5$, $N = 16$, $M = 200$).

\begin{itemize}
    \item \textbf{Gaussian Random Walk} The noise evolves under the uncontrolled OU reference dynamics with the same norm projection. This baseline isolates the effect of random exploration without reward-driven control. It corresponds to the ``Uncontrolled'' baseline in the main paper.
    \item \textbf{ZeNO + Tangent Space} Instead of explicit normalization, both the Gaussian perturbations and the reward-weighted control are restricted to the sphere $\mathcal{S} = \{\vz : \|\vz\| = \sqrt{d}\}$ around which $\gN(0,\rmI)$ concentrates. Specifically, each vector $\vv$ is projected onto the tangent space at the current state $\vz$ via
    \begin{equation}
        \mathrm{P}_{\vz}(\vv) = \vv - \frac{\langle \vv, \vz \rangle}{\|\vz\|^2}\vz,
    \end{equation}
    and the update is applied through a single exponential-map step on $\mathcal{S}$,
    \begin{equation}
        \mathrm{Exp}_{\vz}(\vv) = \cos\left(\frac{\|\vv\|}{\sqrt{d}}\right)\vz + \sqrt{d}\,\sin\left(\frac{\|\vv\|}{\sqrt{d}}\right)\frac{\vv}{\|\vv\|}.
    \end{equation}
    \item \textbf{ZeNO w/o normalization} ZeNO with the norm projection (line 12 of Algorithm~\ref{alg:method}) removed, so that the noise evolves under the controlled OU update in \eqref{eq:disc} alone.
\end{itemize}

\begin{table}[h]
\centering
\caption{Comparison with simpler noise-space search schemes on PickScore with SDXL-Turbo under the same evaluation budget. \textbf{Bold} indicates the best, and \underline{underline} indicates the second-best result.}
\label{tab:perturbation_baselines}
\resizebox{0.75\textwidth}{!}{
\begin{tabular}{lcccc}
\toprule
Method & \makecell{Gaussian\\random walk} & \makecell{ZeNO\\(tangent space)} & \makecell{ZeNO\\w/o normalization} & ZeNO \\
\midrule
PickScore $(\uparrow)$ & 21.65 & \underline{23.40} & 23.01 & \textbf{23.43} \\
\bottomrule
\end{tabular}
}
\end{table}

The Gaussian random walk performs substantially worse than all ZeNO variants, showing that the reward improvement of ZeNO does not arise from random exploration of the noise space alone. Rather, it comes from the reward-weighted control estimated from zeroth-order evaluations.

Removing the normalization still improves over the random walk, which again confirms the benefit of the control update. However, during optimization, the letent gradually drifts away from the Gaussian shell without norm-preserving. This supports the practical design choice, where renormalization keeps the noise compatible with the pre-trained generator throughout optimization. 

\section{Results on Additional Human Preference Reward}
\label{app:hpsv3}

To examine whether ZeNO remains effective under more recent and stronger human preference models, we additionally evaluate on HPSv3~\citep{ma2025hpsv3widespectrumhumanpreference}, a wide-spectrum human preference score trained on large-scale preference data spanning diverse generative models and image sources. We use SDXL-Turbo as the one-step generator and follow the same experimental protocol as in Section~\ref{subsec:exp_img}, with $\beta=0.01$, $\eta=1.5$, $N=16$, and $M=200$. For the gradient-based baselines (ORIGEN and $\Psi$-Sampler), we match the number of generator calls, as in the main experiments.

\begin{table}[h]
\centering
\caption{Results on HPSv3 with SDXL-Turbo. \textbf{Bold} indicates the best, and \underline{underline} indicates the second-best result.}
\label{tab:hpsv3}
\resizebox{0.75\textwidth}{!}{
\begin{tabular}{lccccc}
\toprule
Method & Uncontrolled & Best-of-$N$ & ORIGEN & $\Psi$-Sampler & ZeNO (ours) \\
\midrule
HPSv3 $(\uparrow)$ & 9.804 & 13.693 & 9.570 & \textbf{14.271} & \underline{14.186} \\
\bottomrule
\end{tabular}
}
\end{table}

As shown in Table~\ref{tab:hpsv3}, ZeNO substantially improves HPSv3 over the uncontrolled baseline and outperforms Best-of-$N$ under the same generator call budget, indicating that iterative noise optimization extracts more reward improvement than independent sampling. ZeNO achieves performance comparable to $\Psi$-Sampler, with a gap of less than 0.1, even though $\Psi$-Sampler additionally requires backpropagation through the generator and the reward model. In contrast, ORIGEN falls below the uncontrolled baseline, consistent with our observations on LCM in Table~\ref{tab:differentiable_rewards}. Overall, these results confirm that ZeNO generalizes to recent human preference models and remains competitive with gradient-based methods while using only zeroth-order reward evaluations.

\section{Diversity Analysis}
\label{app:diversity_analysis}

ZeNO is designed to improve reward while preserving the diversity of the reference generator.
From the Langevin interpretation in Proposition~\ref{prop:langevin}, the controlled noise dynamics does not aim to collapse all samples to a single reward-maximizing point.
Instead, it moves the noise toward a reward-tilted distribution,
\begin{equation}
    p^\star_\kappa(\vz) \propto p(\vz)\exp\big(\kappa V(\vz)\big),
\end{equation}
where the reference prior $p(\vz)$ preserves the support and diversity of the original generator, while the exponential value tilt increases the probability of high-reward regions. 
Therefore, although the generated reward distribution is expected to shift toward higher reward samples, the resulting sample distribution should still retain diversity as long as $\kappa$ is not too large.
This behavior is consistent with our toy experiment in Figure~\ref{fig:toy}, where ZeNO closely matches the target tilted distribution, while gradient-based updates concentrate probability mass on a subset of modes.

To empirically verify this behavior in high-dimensional image generation, we measure diversity throughout ZeNO optimization using the Vendi Score~\citep{friedman2023vendi} following~\citet{kim2025diverse}. Given a set of generated images $\{x_i\}_{i=1}^n$, we compute image embeddings and form a normalized pairwise similarity matrix $K$. The Vendi Score is defined as the exponential entropy of the eigenvalue distribution of $K$,
\begin{equation}
    \mathrm{VS}(\{x_i\}_{i=1}^n)
    =
    \exp\left(
    -\sum_j \lambda_j \log \lambda_j
    \right),
\end{equation}
where $\{\lambda_j\}$ are the normalized eigenvalues of $K$. A higher Vendi Score indicates a larger effective number of distinct samples, while a decreasing score indicates reduced diversity.

Figure~\ref{fig:vendi_analysis} shows the Vendi Score over ZeNO iterations. We observe a small decrease at the early stage of optimization, which is expected because the reward tilt shifts the initial distribution toward higher-reward regions. Importantly, however, the diversity does not continue to decrease. Instead, the Vendi Score quickly saturates and remains stable over the rest of the optimization trajectory. This suggests that ZeNO improves reward by redistributing probability mass according to the tilted distribution rather than collapsing the generator output to a narrow set of samples. In other words, ZeNO preserves meaningful sample diversity while steering the noise distribution toward higher-reward regions.

\begin{figure}[t]
    \centering
    \begin{minipage}[t]{0.43\linewidth}
        \centering
        \includegraphics[width=\linewidth]{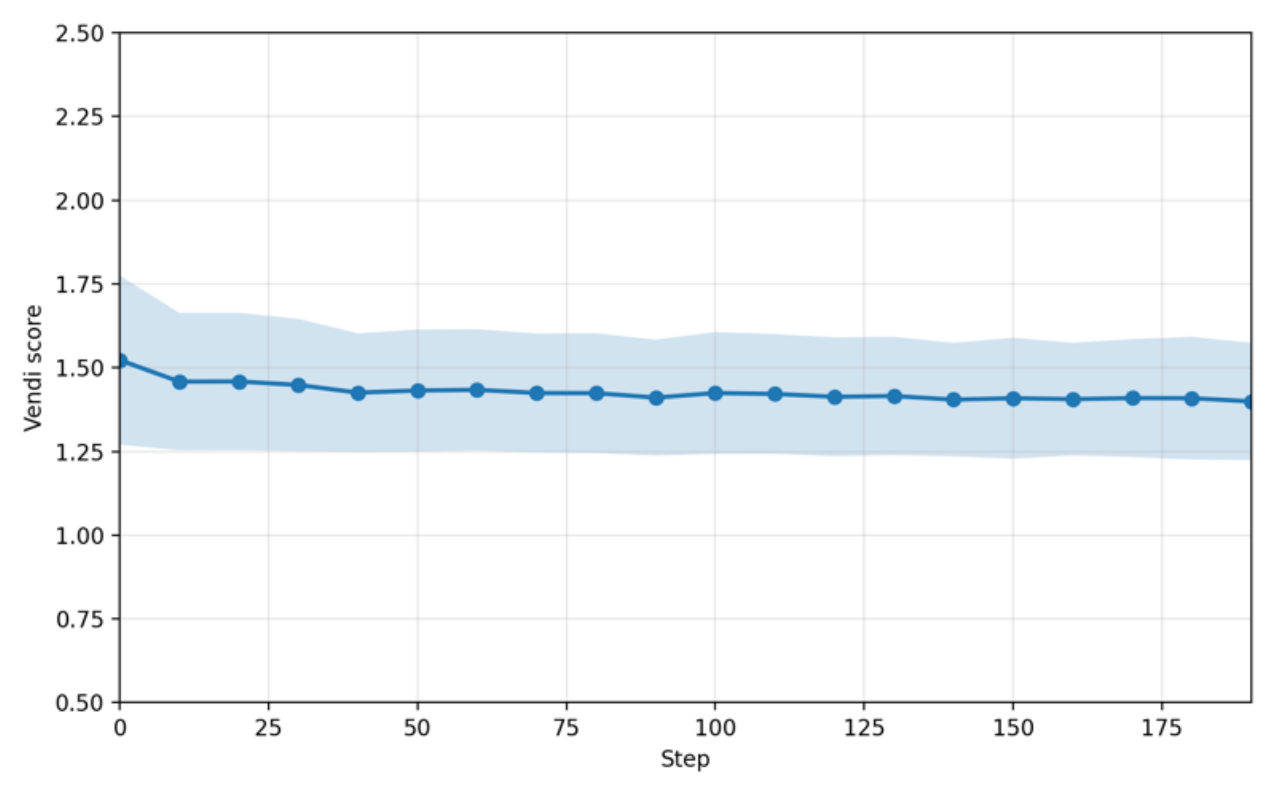}
        \caption*{(a) Vendi Score}
    \end{minipage}
    \hfill
    \begin{minipage}[t]{0.53\linewidth}
        \centering
        \includegraphics[width=\linewidth]{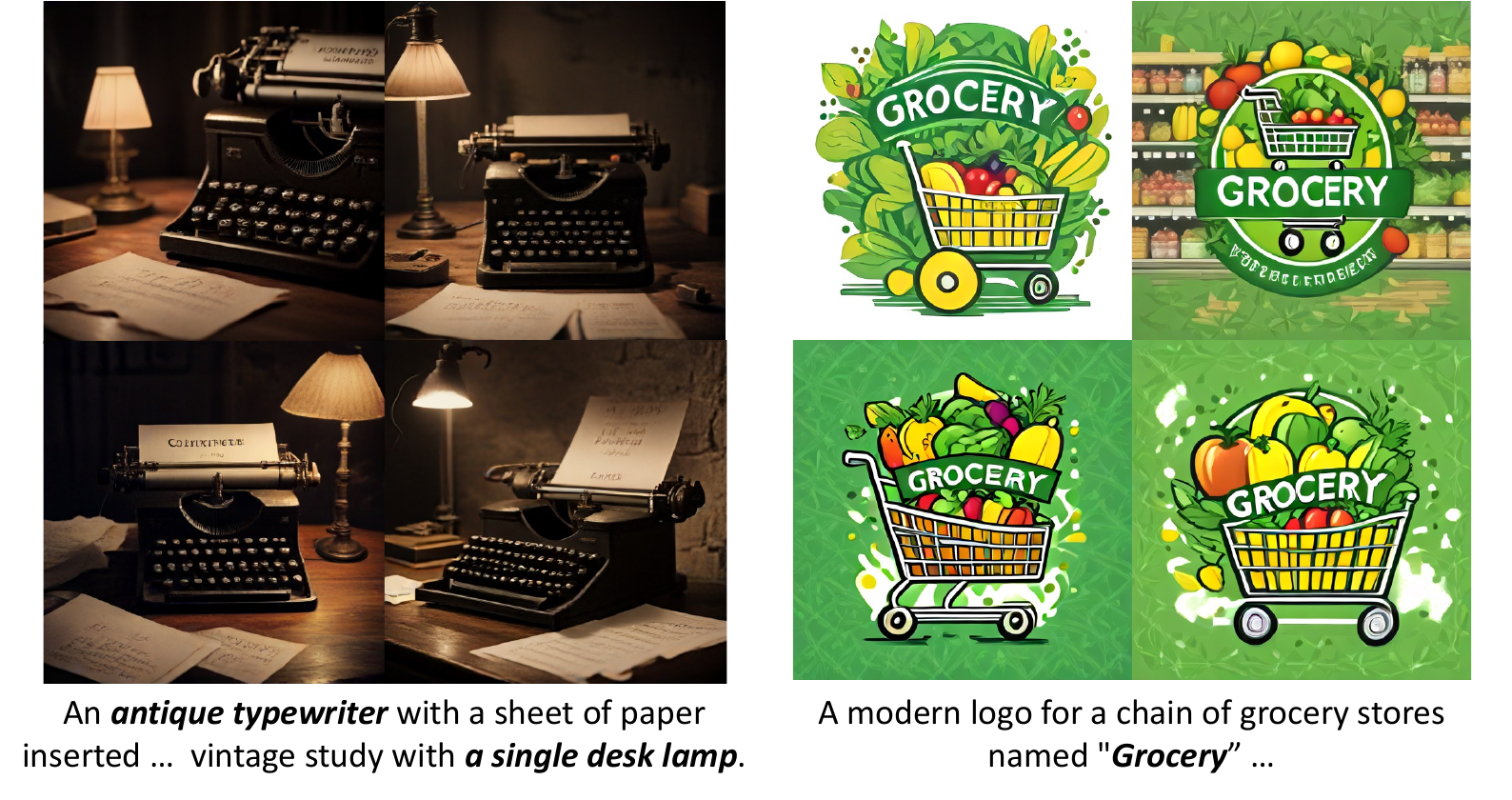}
        \caption*{(b) Generated examples}
    \end{minipage}
    \caption{
    Diversity analysis of ZeNO.
    Left: Vendi Score over ZeNO iterations. The score slightly decreases during the initial reward-guided optimization phase, but quickly saturates and remains stable across subsequent iterations.
    Right: representative generated samples during optimization.
    These results indicate that ZeNO preserves sample diversity and does not collapse to a narrow set of modes while shifting the generation distribution toward higher-reward regions.
    Shaded regions indicate variability across samples.
    }
    \label{fig:vendi_analysis}
\end{figure}

\section{Wall-Clock Analysis and Complementarity with Best-of-$N$}
\label{app:wallclock}

While the main experiments compare methods under a fixed number of
generator/reward evaluations, Best-of-$N$ (BoN) is fully parallelizable,
whereas ZeNO performs $M$ sequential update rounds. We therefore compare
the two under a fixed wall-clock budget on a single RTX 4090, using
SDXL-Turbo with PickScore.

\begin{table}[h]
\centering
\caption{PickScore under a fixed wall-clock budget (single RTX 4090, SDXL-Turbo).}
\label{tab:wallclock}
\begin{tabular}{lccccccc}
\toprule
Wall-clock (s) & 0 & 50 & 100 & 150 & 200 & 250 & 300 \\
\midrule
Best-of-$N$ & 22.81 & 23.05 & 23.08 & \textbf{23.20} & 23.20 & 23.20 & 23.20 \\
ZeNO        & 22.50 & 23.08 & 23.15 & 23.31 & 23.38 & \textbf{23.52} & 23.52 \\
\bottomrule
\end{tabular}
\end{table}

Under a fixed hardware budget, the per-round cost of ZeNO is comparable to
that of BoN, since each round evaluates $N$ particles in parallel. BoN starts
higher because it selects the best of many random samples, whereas ZeNO starts
from a single uncontrolled sample. ZeNO quickly catches up, and after roughly
150\,s BoN saturates while ZeNO continues to improve, reaching 23.52 within the
same budget. With more memory or multiple GPUs, BoN could reduce its latency
further; however, ZeNO's particle evaluations are equally parallelizable.

\section{Implementation Details}
\label{sec:detail}

All experiments are conducted using NVIDIA RTX 3090 and RTX 4090 GPUs with 24GB of memory.

\subsection{Image Generation}
\label{subsec:image_setting}

\paragraph{Figure captions}
The images in Figure~\ref{fig:main} were generated using SDXL-Turbo,
and the corresponding text prompts are listed below, starting from the upper-left and proceeding counter-clockwise.
\begin{enumerate}
    \item ``A teddy bear wearing holding a whiteboard, `ZeNO' written on it.''
    \item ``A beautifully decorated birthday cake with smooth blue icing, the letters `Happy 30th Jake' elegantly spelled out on top, surrounded by colorful candles and sparkling decorations, set on a white tablecloth in a cozy living room.''
    \item ``Studio shot of intricate shoe sculptures crafted from vibrant colored wires, with the text `rte' prominently displayed on a clean, white background.''
    \item ``A campfire in a dense forest, with smoke curling upwards and forming the words `Send Help' against a twilight sky, partially obscured by the treetops.''
    \item ``An astronaut proudly displays the `Mission Success' patch on their spacesuit, standing against the backdrop of a stunning Earthrise over the lunar horizon, with the vast cosmos stretching beyond.''
\end{enumerate}

\paragraph{CFG scale}
We used a CFG scale of 0.0 for SDXL-Turbo and DMD2, and 8.0 for LCM, following the default recommended settings for each model.

\subsection{Protein Structure Generation}
\label{subsec:protein_setting}

\paragraph{Initial-frame representation.}
For protein structure generation, we apply ZeNO to a one-step Riemannian MeanFlow (RMF)
backbone generator. Following frame-based protein generative models, each backbone of length
$N$ is represented as residue-wise rigid frames
$x = \{(R_i, t_i)\}_{i=1}^{N} \in \mathrm{SE}(3)^N$, where $R_i \in \mathrm{SO}(3)$ and
$t_i \in \mathbb{R}^3$. The RMF generator maps an initial frame sample $x$ to a backbone
$y = \Phi_\theta(x)$, and $\Phi_\theta$ is kept frozen throughout ZeNO optimization.

\paragraph{Reference dynamics on protein frames.}
Since $\mathrm{SO}(3)$ is not closed under addition, we replace the Euclidean OU step in ~\eqref{eq:disc} with a prior-preserving counterpart on frames. For step sizes
$\beta, \beta_R \in (0,1)$, one step of the reference kernel reads
\begin{align}
    t_i^{k+1} &= \sqrt{1-\beta}\, t_i^k + \sqrt{\beta}\,\sigma_t\, \epsilon_i^k,
    \qquad \epsilon^k \sim \mathcal{N}_0(0, I), \\
    R_i^{k+1} &= \exp\!\big(\sqrt{\beta_R}\, \widehat{\omega_i^k}\big) R_i^k,
    \qquad \omega_i^k \sim \mathcal{N}(0, I_3),
\end{align}
where $\sigma_t$ is the standard deviation of the translation prior, $\mathcal{N}_0(0, I)$
denotes the standard Gaussian projected onto the zero-center-of-mass subspace of
$\mathbb{R}^{3N}$, and $\widehat{\cdot}$ maps a vector in $\mathbb{R}^3$ to its
skew-symmetric matrix in $\mathfrak{so}(3)$. The translation step preserves the centered
Gaussian prior since $(\sqrt{1-\beta})^2 + \beta = 1$, and the rotation step preserves the
uniform (Haar) prior on $\mathrm{SO}(3)^N$. Hence, as in the image setting, the uncontrolled
kernel leaves the RMF source distribution invariant.

\paragraph{Control estimation and update.}
At each iteration $k$, we draw $K$ one-step proposals from the reference kernel,
\begin{equation}
    \tilde t_i^{(j)} = \sqrt{1-\beta}\,t_i^k + \sqrt{\beta}\,\sigma_t\, v_i^{(j)}, \qquad
    \tilde R_i^{(j)} = \exp\!\big(\sqrt{\beta_R}\, \widehat{\omega_i^{(j)}}\big) R_i^k,
\end{equation}
with $v^{(j)} \sim \mathcal{N}_0(0, I)$ and $\omega_i^{(j)} \sim \mathcal{N}(0, I_3)$,
evaluate $r_j = r(\Phi_\theta(\tilde x^{(j)}))$, where $r$ is the black-box designability
reward defined below, and form advantages
$a_j = r_j - \frac{1}{K}\sum_{j'} r_{j'}$. As in Algorithm~\ref{alg:method}, we average the advantage-weighted perturbations, normalizing the translation and rotation components by their respective noise scales:
\begin{equation}
    u_{t,i} = \frac{1}{\sqrt{\beta}\,\sigma_t K}\sum_{j=1}^{K} a_j\, v_i^{(j)}, \qquad
    u_{R,i} = \frac{1}{\sqrt{\beta_R} K}\sum_{j=1}^{K} a_j\, \omega_i^{(j)}.
\end{equation}
The frames are then updated by the controlled reference kernel,
\begin{align}
    t_i^{k+1} &= \sqrt{1-\beta}\, t_i^k + \sqrt{\beta}\,\sigma_t\, \epsilon_i^k
    + \eta\, \bar u_{t,i}, \\
    R_i^{k+1} &= \exp\!\big(\eta\, \widehat{\bar u_{R,i}}
    + \sqrt{\beta_R}\, \widehat{\omega_i^k}\big) R_i^k,
\end{align}
with fresh noise $\epsilon^k \sim \mathcal{N}_0(0, I)$ and $\omega_i^k \sim \mathcal{N}(0, I_3)$,
followed by re-centering the translations to zero center of mass. Here, $\eta$ is the step
size, and $\bar u_{t,i}$ and $\bar u_{R,i}$ denote the controls clipped so that each
per-residue update $\eta\,\bar u_{\cdot,i}$ has norm at most $\tau_t$ and $\tau_R$,
respectively, for numerical stability.

\paragraph{Designability reward.}
We use a black-box designability reward based on an inverse-folding and refolding self-consistency protocol. Given a generated backbone $y=\Phi_\theta(x)$, we first use ProteinMPNN to design $M$ amino-acid sequences conditioned on $y$:
\begin{equation}
    s_m \sim p_{\mathrm{MPNN}}(s \mid y),
    \qquad
    m=1,\dots,M.
\end{equation}
Each designed sequence is then refolded using ESMFold:
\begin{equation}
    \hat y_m = \mathrm{ESMFold}(s_m).
\end{equation}
We measure the self-consistency between the generated backbone $y$ and the refolded structure $\hat y_m$ using the self-consistency RMSD,
\begin{equation}
    d_m = \mathrm{scRMSD}(y,\hat y_m).
\end{equation}
Following standard designability evaluation, we summarize the backbone-level self-consistency by the best refolded sequence,
\begin{equation}
    d_{\mathrm{sc}}(y)
    =
    \min_{m=1,\dots,M}
    \mathrm{scRMSD}(y,\hat y_m).
\end{equation}
A generated backbone is considered designable if
\begin{equation}
    d_{\mathrm{sc}}(y) < 2\text{\AA}.
\end{equation}
For reward-guided optimization, we convert this self-consistency score into a scalar reward. In the simplest form, we use the negative self-consistency RMSD,
\begin{equation}
    r(y)
    =
    - d_{\mathrm{sc}}(y),
\end{equation}
so that higher reward corresponds to lower refolding error and thus higher designability.
This reward is non-differentiable with respect to the RMF initial frames because it involves ProteinMPNN sequence sampling, ESMFold refolding, structural alignment, and RMSD computation. ZeNO only requires reward evaluations and therefore can directly optimize this designability objective without backpropagating through ProteinMPNN, ESMFold, or the self-consistency metric.

\newpage
\section{Additional Qualitative Results}
\label{sec:additional_results}

\subsection{Image Results}
\label{sec:additional_image_qual_results}

\begin{figure}[h]
    \centering
    \includegraphics[width=0.9\linewidth]{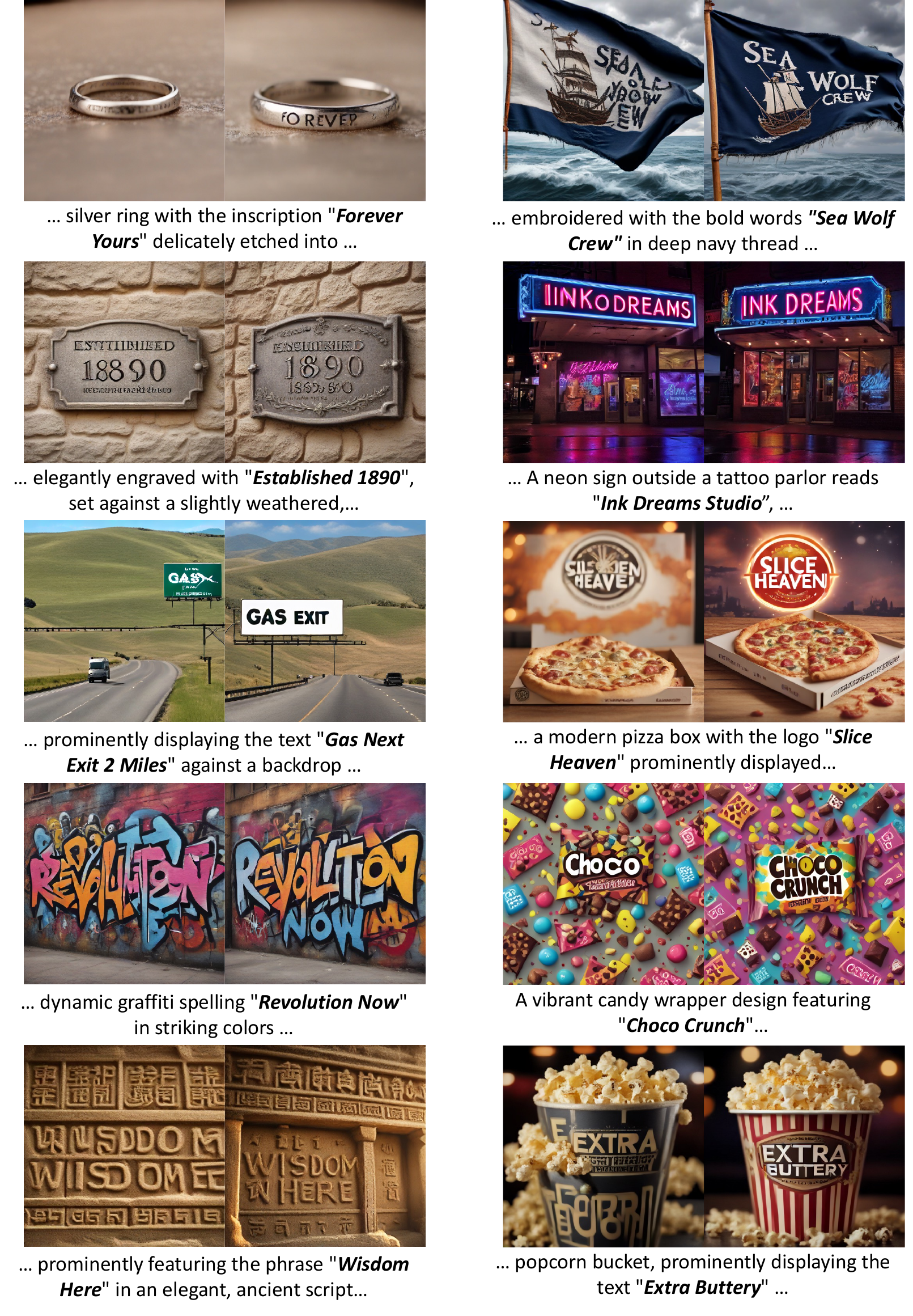}
    \caption{Images generated by 1-step SDXL-Turbo with ZeNO, using PickScore as the reward signal.}
    \label{fig:additional_turbo_pickscore}
\end{figure}

\begin{figure}[h]
    \centering
    \includegraphics[width=0.9\linewidth]{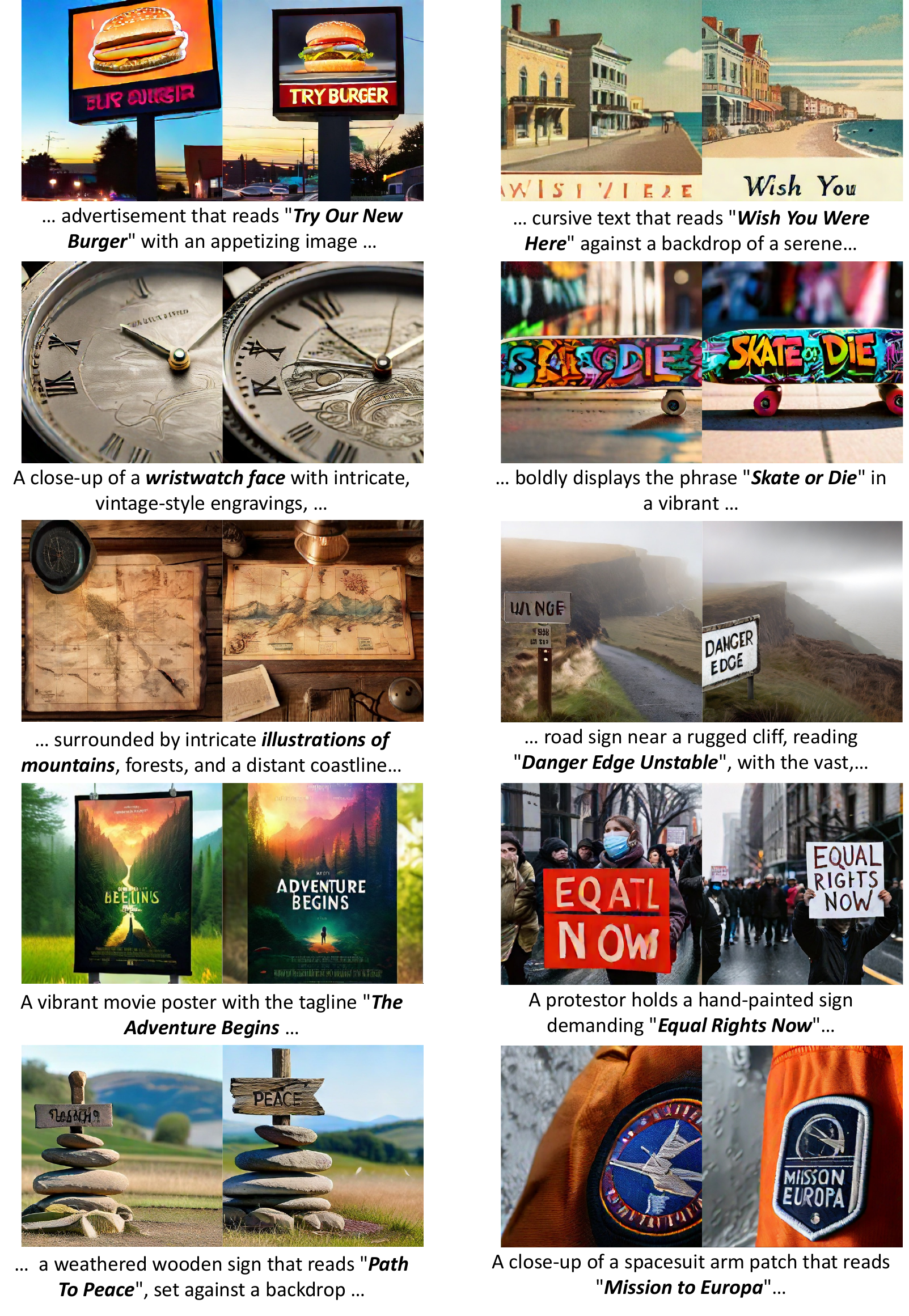}
    \caption{Images generated by 1-step SDXL-DMD2 with ZeNO, using PickScore as the reward signal.}
    \label{fig:additional_dmd2_pickscore}
\end{figure}

\begin{figure}[t]
    \centering
    \includegraphics[width=0.9\linewidth]{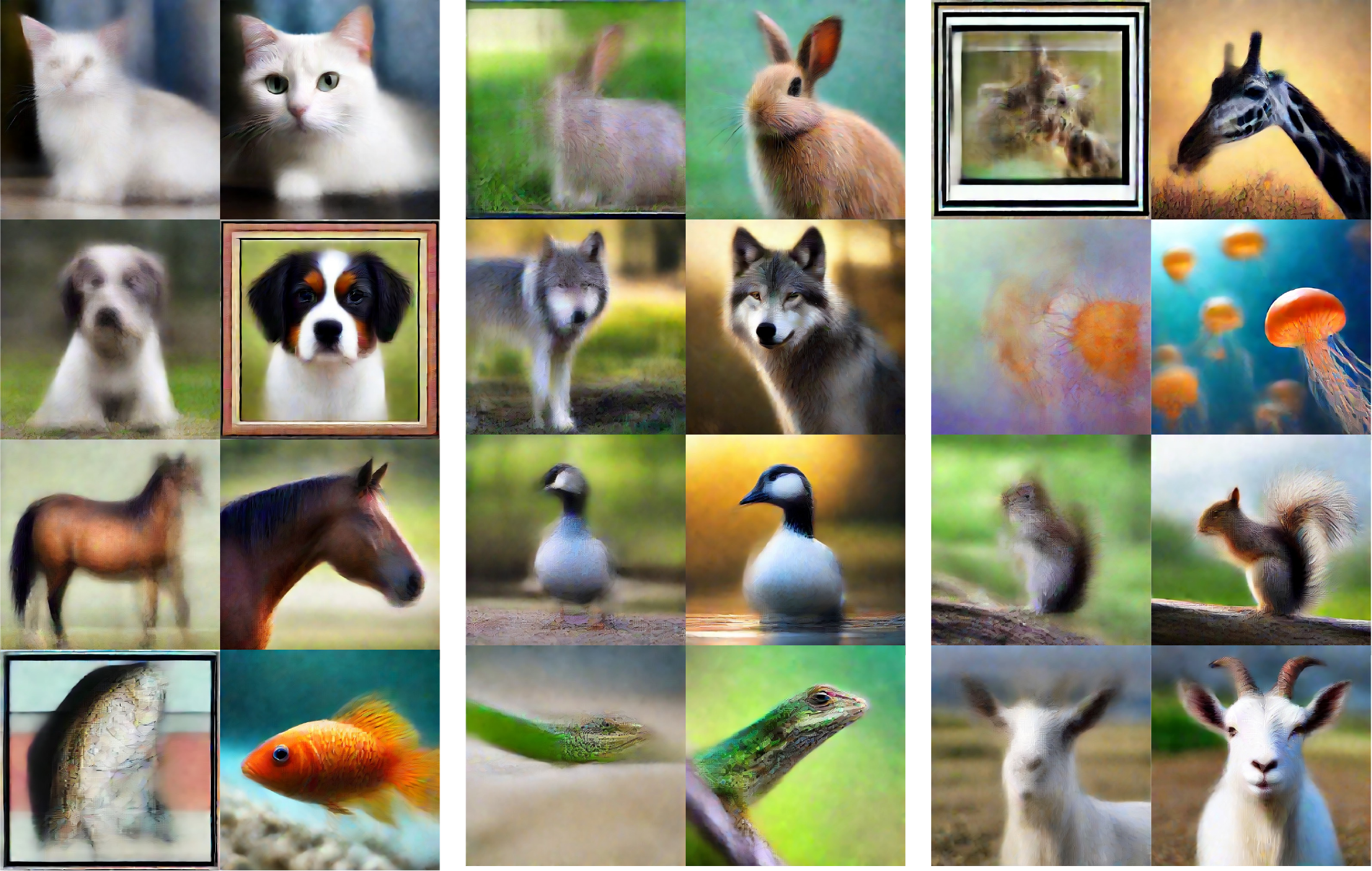}
    \caption{Images generated by 1-step LCM with ZeNO, using aesthetic score as the reward signal.}
    \label{fig:additional_lcm_aesthetic}
\end{figure}

\begin{figure}[t]
    \centering
    \includegraphics[width=0.9\linewidth]{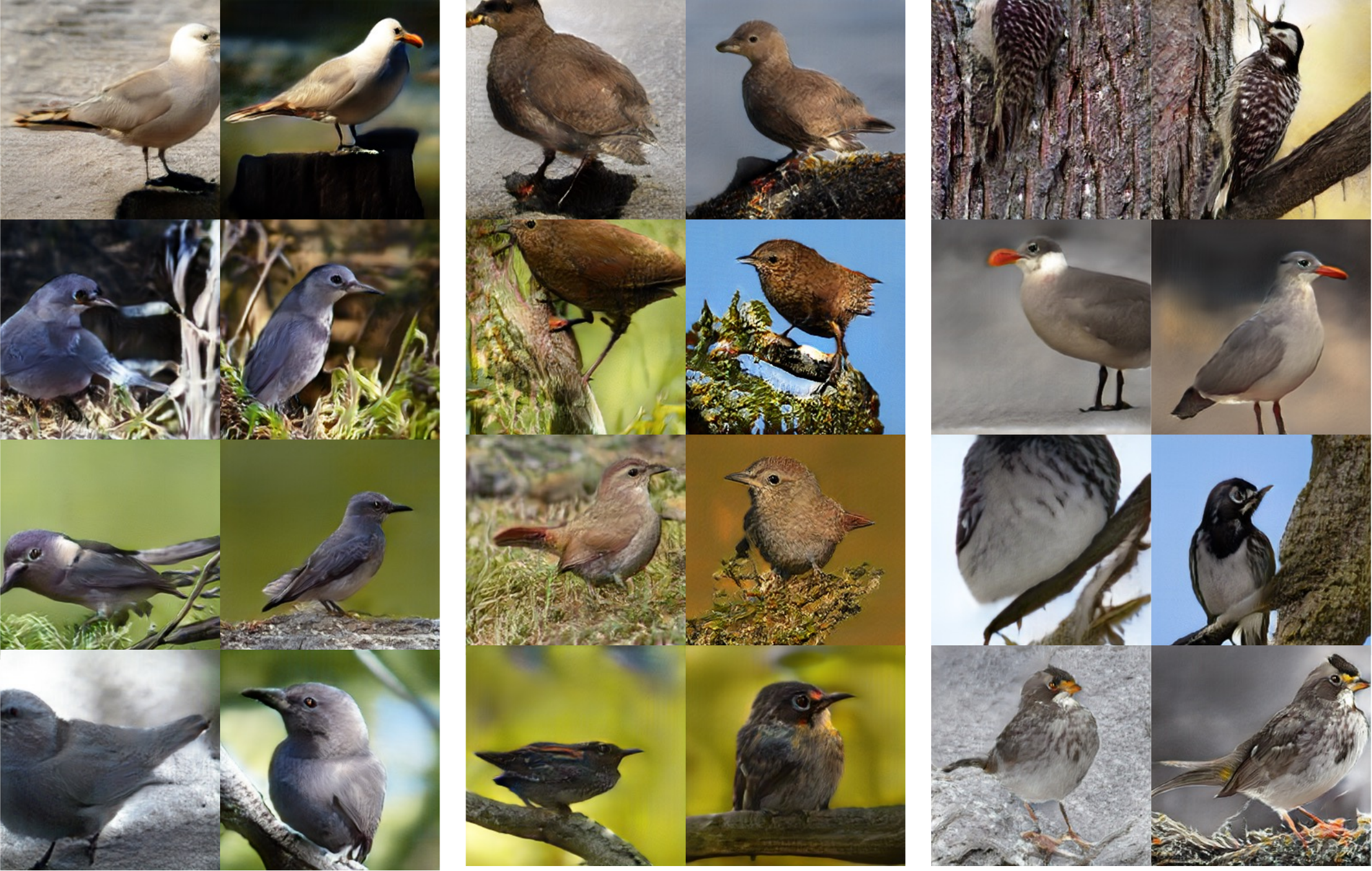}
    \caption{Images generated by DF-GAN with ZeNO, using aesthetic score as the reward signal.}
    \label{fig:additional_dfgan_aesthetic}
\end{figure}

\clearpage
\subsection{Protein Results}
\label{sec:additional_protein_qual_results}
We provide additional uncurated qualitative results for 30 protein samples, showing that the per-sample scRMSD decreases over ZeNO iterations.

\begin{figure}[h]
    \centering
    \includegraphics[width=0.9\linewidth]{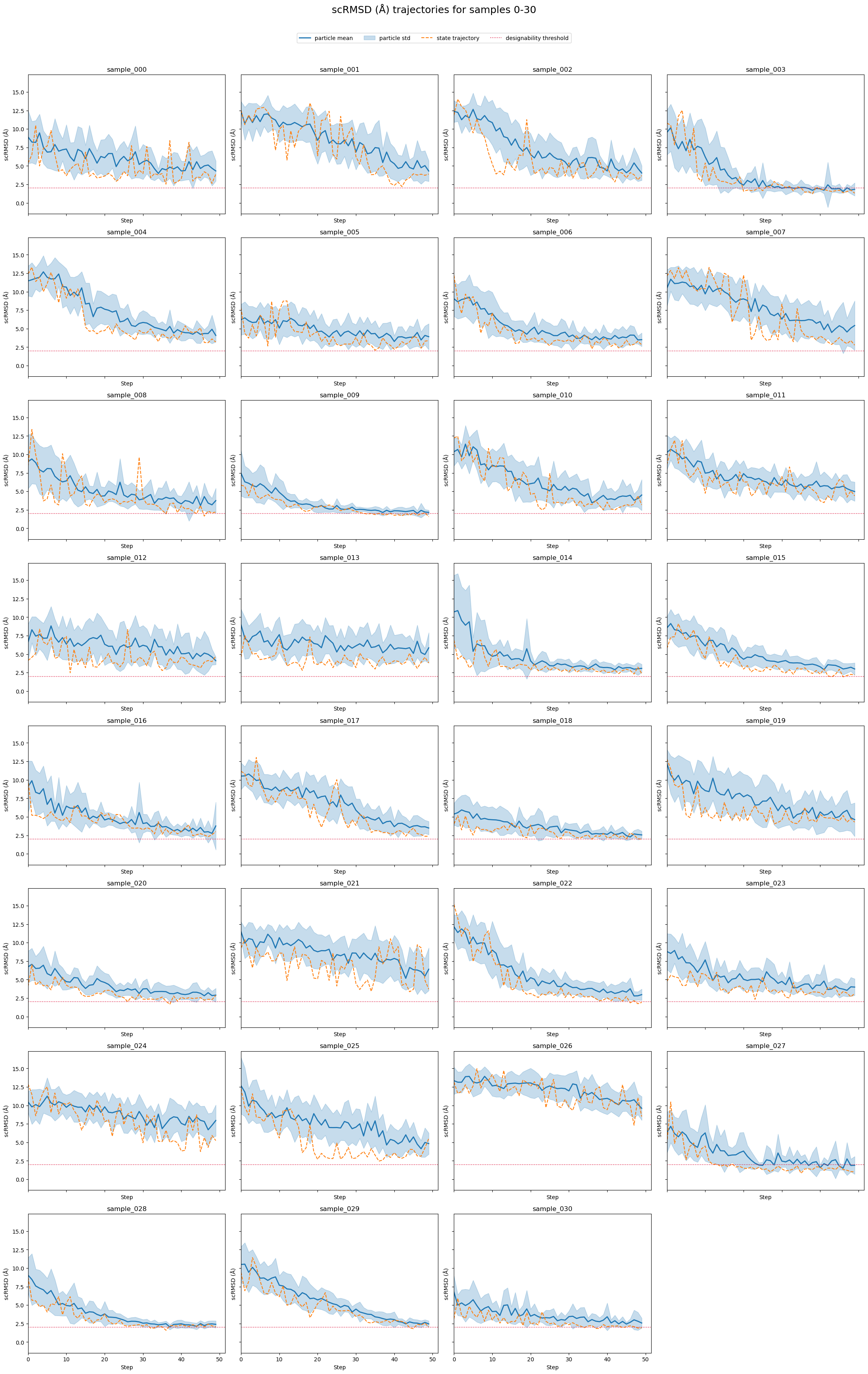}
    \caption{Step-wise scRMSD trajectories for each sample generated by the RMF/S model.}
    \label{fig:rmf_s_per_sample}
\end{figure}

\begin{figure}[h]
    \centering
    \includegraphics[width=0.95\linewidth]{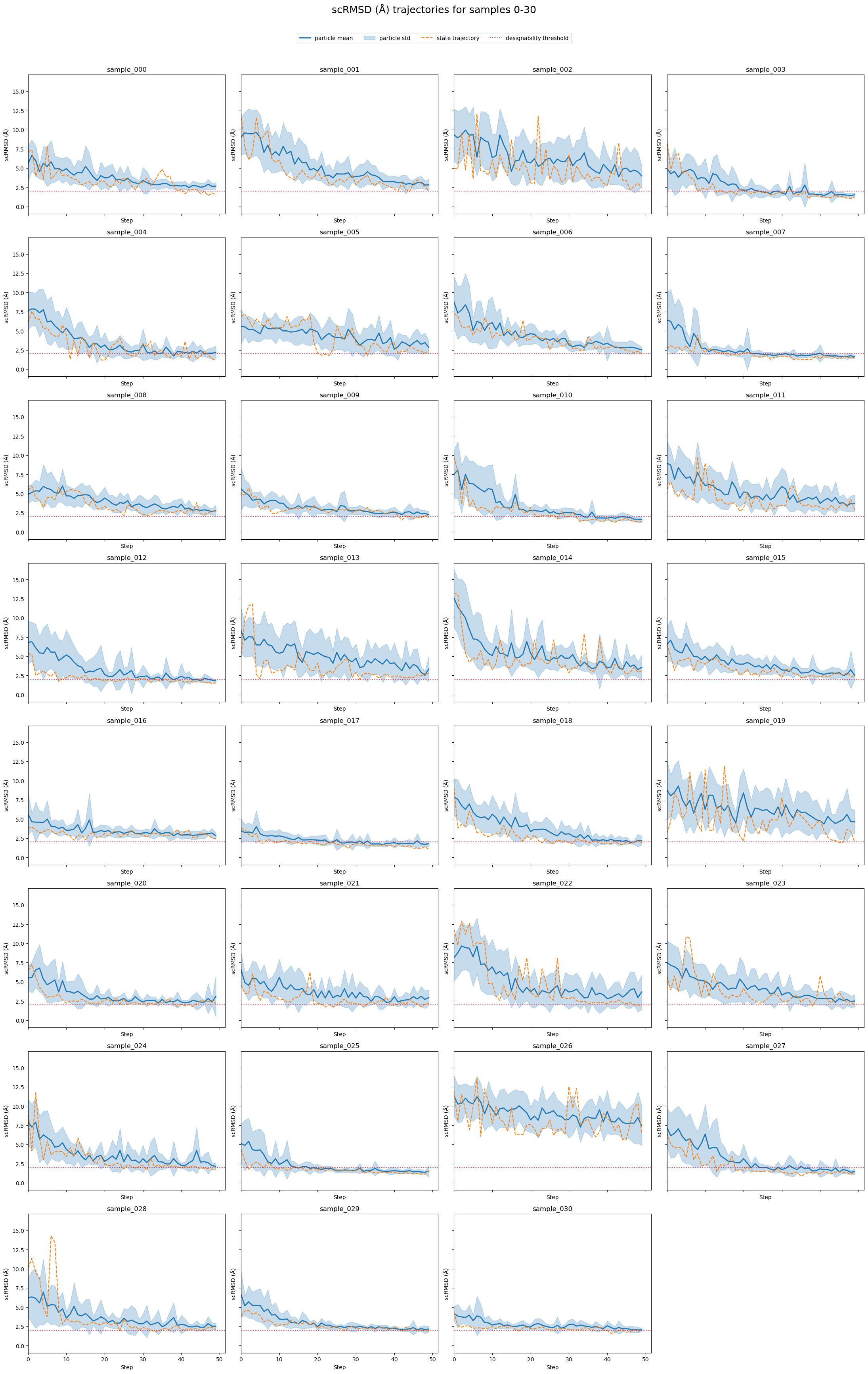}
    \caption{Step-wise scRMSD trajectories for each sample generated by the RMF/M model.}
    \label{fig:rmf_m_per_sample}
\end{figure}

\end{document}